%% file: MAIN.tex
\documentclass{ecai} 
\usepackage{latexsym}
\usepackage{amssymb}
\usepackage{amsmath}
\usepackage{amsthm}
\usepackage{booktabs}
\usepackage{enumitem}
\usepackage{graphicx}
\usepackage{color}
\usepackage{multicol}
\usepackage{multirow}
\usepackage{subfigure}
\usepackage{adjustbox}
\usepackage{latexsym}
\usepackage{amsfonts,amsmath,amssymb,amsthm}
\usepackage{enumitem}
\usepackage{adjustbox}
\usepackage{booktabs}
\usepackage{multicol,multirow}
\usepackage{soul}
\usepackage{algorithm}
\usepackage{algorithmic}
\usepackage{mathtools}

\definecolor{darkmagenta}{rgb}{0.6, 0.0, 0.5} % slightly dark magenta
\definecolor{darkblue}{rgb}{0.0, 0.0, 0.5} % dark blue
\definecolor{darkgreen}{rgb}{0.0, 0.69, 0.0} % dark green

\newcommand{\llamaThreeSeventyB}{\textcolor{darkmagenta}{\texttt{LLaMA-3.1-70B}}}
\newcommand{\llamanvidia}{\textcolor{darkmagenta}{\texttt{LLaMA-3.1-70B-NV}}}
\newcommand{\qwen}{\textcolor{darkmagenta}{\texttt{Qwen2.5-72B}}}
\newcommand{\gptMini}{\textcolor{darkmagenta}{\texttt{gpt-4o-mini}}}
\newcommand{\gptFour}{\textcolor{darkmagenta}{\texttt{gpt-4o}}}
\newcommand{\gptThreeFive}{\textcolor{darkmagenta}{\texttt{gpt-3.5-turbo-16k}}}
\newcommand{\llamaEight}{\textcolor{darkmagenta}{\texttt{LLaMA-3.1-8B}}}
\newcommand{\qwenOneFour}{\textcolor{darkmagenta}{\texttt{Qwen2.5-14B}}}
\newcommand{\mistralSeven}{\textcolor{darkmagenta}{\texttt{Mistral-7B}}}

\newcommand{\llamaThree}{\textcolor{darkmagenta}{\texttt{LLaMA-3.2-3B}}}
\newcommand{\qwenThree}{\textcolor{darkmagenta}{\texttt{Qwen2.5-3B}}}

\newcommand{\greedy}{\textcolor{darkblue}{\texttt{Greedy}}}
\newcommand{\optimization}{\textcolor{darkblue}{\texttt{Optimization}}}
\newcommand{\hillclimb}{\textcolor{darkblue}{\texttt{Hill Climb}}}

\newcommand{\ppi}{\textsl{PPI}}
\newcommand{\vrpi}{\textsl{VRPI}}
\newcommand{\nlvpi}{\textsl{NLVPI}}

\newtheorem{proposition}{Proposition}
\renewcommand{\paragraph}[1]{\smallskip\noindent\textbf{#1}~}

\newcommand{\BibTeX}{B\kern-.05em{\sc i\kern-.025em b}\kern-.08em\TeX}

\begin{document}

\begin{frontmatter}

%%% Use this command to specify your submission number.
%%% In doubleblind mode, it will be printed on the first page.

\paperid{123} 

%%% Use this command to specify the title of your paper.

\title{LLMs for Resource Allocation: A Participatory Budgeting Approach to Inferring Preferences}

%%% Use this combinations of commands to specify all authors of your 
%%% paper. Use \fnms{} and \snm{} to indicate everyone's first names 
%%% and surname. This will help the publisher with indexing the 
%%% proceedings. Please use a reasonable approximation in case your 
%%% name does not neatly split into "first names" and "surname".
%%% Specifying your ORCID digital identifier is optional. 
%%% Use the \thanks{} command to indicate one or more corresponding 
%%% authors and their email address(es). If so desired, you can specify
%%% author contributions using the \footnote{} command.

% \author[]{Paper ID 6529}
\author[A]{\fnms{Sankarshan}~\snm{Damle}\orcid{0000-0003-1460-6102}\thanks{Corresponding Author. Email: \texttt{sankarshan.damle@epfl.ch}. Published in the Proceedings of the 28th European Conference on Artificial Intelligence (ECAI 2025). This is the full version.}}
\author[A]{\fnms{Boi}~\snm{Faltings}\orcid{0000-0002-7188-7230}}
% \author[B,C]{\fnms{Third}~\snm{Author}\orcid{....-....-....-....}} 

\address[A]{Artificial Intelligence Laboratory (LIA), EPFL}
% \address[B]{Short Affiliation of Second Author and Third Author}
% \address[C]{Short Alternate Affiliation of Third Author}

%%% Use this environment to include an abstract of your paper.

%%%% 7 pages + 1 for reference
%%%% table captions at the top, figure captions bottom
\begin{abstract}
Large Language Models (LLMs) are increasingly expected to handle complex decision-making tasks, yet their ability to perform structured resource allocation remains underexplored. Evaluating their reasoning is also difficult due to data contamination and the static nature of existing benchmarks. We present a dual-purpose framework leveraging Participatory Budgeting (PB) both as (i) a practical setting for LLM-based resource allocation and (ii) an adaptive benchmark for evaluating their reasoning capabilities. We task LLMs with selecting project subsets under feasibility (e.g., budget) constraints via three prompting strategies: greedy selection, direct optimization, and a hill-climbing–inspired refinement. We benchmark LLMs' allocations against a utility-maximizing oracle. Interestingly, we also test whether LLMs can infer structured preferences from natural-language voter input or metadata, without explicit votes. By comparing allocations based on inferred preferences to those from ground-truth votes, we evaluate LLMs’ ability to extract preferences from open-ended input. Our results underscore the role of prompt design and show that LLMs hold promise for mechanism design with unstructured inputs.
\end{abstract}
\end{frontmatter}

\section{Introduction}
\label{sec::intro}
\input{sections/intro}

\section{Related Work}
\label{sec::rw}
\input{sections/rw}

\section{Background}
\label{sec::bg}
\input{sections/background}

\section{Methodology}
\label{sec::method}
\input{sections/methodology}

\section{Results}
\label{sec::results}
\input{sections/results}

\section{Conclusion}

In this paper, we demonstrate that PB mechanisms can serve as a dynamic benchmark for evaluating the reasoning and decision-making abilities of LLMs. We assess state-of-the-art LLMs based on their performance on real-world PB instances, emphasizing their efficacy and instruction-following ability. We compare performance across different instance types to also showcase LLMs' adaptability in providing viable solutions in setups with natural language or omitted preferences. Our work underscores the value of mechanism design as a benchmark for LLM evaluation.

\paragraph{Limitations \& Future Work.}
Here we discuss limitations and directions for future work. First, the prompts used may not be fully optimized, which could affect model performance. Similarly, the \texttt{NLVPI} representation may not be fully refined, meaning performance could improve with better representations. Second, our experiments are constrained by the limited external metadata available, specifically age, sex, and education of voters. More fine-grained voter information could potentially improve LLMs' ability to better infer preferences, leading to more accurate predictions and improving the allocation it outputs. Third, our experiments are also limited by the availability and sample size of real-world PB instances. Lastly, our focus on a single framework (PB) may not capture the full complexity of decision-making systems. Future work could explore additional datasets, scenarios, and mechanisms to evaluate LLMs.

\bibliography{ref}

\newpage
\appendix

\input{sections/appendix}

\end{document}

%% file: sections/intro.tex
The integration of Large Language Models (LLMs)~\citep{achiam2023gpt} into real-world applications has significantly advanced capabilities in complex problem-solving and decision-making, ranging from task automation to enabling autonomous negotiation agents~\citep{li2023camel,microsoft2023copilot,google2023bardupdate,su-etal-2024-language}. Although originally trained via self-supervised learning on large-scale corpora, LLMs now exhibit emergent reasoning abilities that go beyond their initial design goals~\citep{bubeck2023sparks}. Examples include systems such as ChatGPT Plugins~\citep{openai2023chatgptplugins} and AutoGPT~\cite{Significant_Gravitas_AutoGPT}, which demonstrate the ability to perform complex, goal-directed tasks. 

Many real-world decision problems (e.g., scheduling~\citep{pinedo1992scheduling}, logistics~\citep{toth2014vehicle}, and participatory governance~\citep{aziz2021participatory}) can be framed as \textit{resource allocation} tasks, which involve distributing limited resources among competing agents. These tasks are typically formulated as \textit{constrained optimization} problems, subject to feasibility constraints such as budgets, conflicts, or capacities. This setting naturally arises in civic budgeting, humanitarian aid, and multi-agent planning~\citep{aziz2021participatory,damle2019aggregating}. Given their growing reasoning capabilities, a natural question is whether LLMs can act as \textit{social planners}, interpreting task constraints to produce efficient allocations.

In parallel, traditional static benchmarks used to evaluate model performance on fixed tasks are becoming increasingly insufficient~\citep{guo2023evaluating}. They particularly fail to reflect the adaptive nature of real-world decision-making tasks. Moreover, concerns about data contamination -- where benchmark data overlaps with training data -- raise questions about whether LLMs are genuinely reasoning or merely recalling patterns~\citep{singh2024evaluation,xu2024benchmark}. Using LLMs as social planners in resource allocation offers a natural and functional testbed for evaluating their reasoning in dynamic, goal-driven contexts.

\begin{figure}[t]
    \centering 
    \includegraphics[width=\linewidth,trim={20pt 70pt 60pt 70pt},clip]{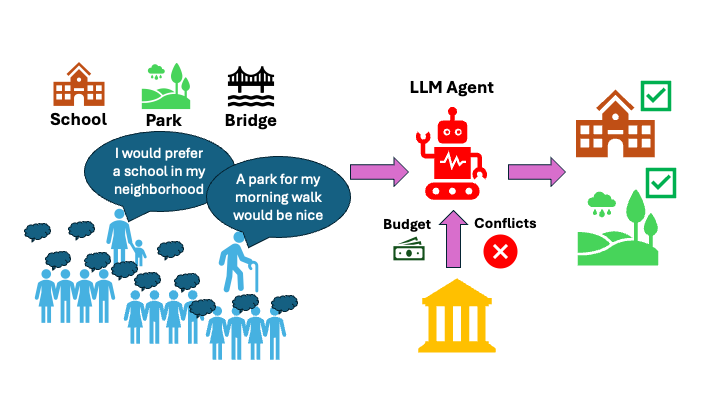} 
    \caption{\textbf{Participatory Budgeting as a Benchmark for LLM Evaluation:} An LLM agent acting as a social planner allocates community resources based on dynamic factors, including community preferences, budget constraints, and project conflicts.} 
    \label{fig:teaser} 
\end{figure}

\paragraph{Mechanism Design-based LLM Reasoning.}
Recent research has explored LLMs’ capabilities for reasoning and decision-making in adaptive environments, focusing on tasks such as negotiation, preference inference, and strategic interactions~\citep{kambhampati2024position}. In this paper, we introduce a \textit{mechanism design}~\citep{narahari2014game} framework to evaluate LLMs as autonomous decision-makers. We focus on \textit{Participatory Budgeting} (PB)~\citep{aziz2021participatory}, a well-established mechanism used by municipalities to allocate resources based on collective preferences~\cite{rocke2014framing}. In PB, participants express their preferences regarding resource allocation, which guides the decision-making process. This task provides a complex real-world environment where the LLM, acting as a {social planner}, allocates resources by interpreting preferences, reasoning under constraints, and optimizing allocations. See Figure~\ref{fig:teaser} for an illustration.

Unlike static benchmarks, PB tasks are inherently adaptive, requiring LLMs to handle evolving factors such as changing budget constraints, project conflicts, or shifting preferences. To elicit allocations, we prompt LLMs using three distinct strategies -- greedy selection, direct optimization, and a hill-climbing–inspired refinement~\cite{peter2021modern} -- each designed to probe different aspects of their reasoning and decision-making. To evaluate LLM performance, we compare the allocations they produce to those generated by a utility-maximizing oracle~\citep{rey2023computational}, providing a structured and meaningful assessment of their decision-making capabilities. 

\paragraph{Inferring Preferences from Natural Language or Community Metadata.}
We also investigate how well LLMs can infer structured preferences from (i) natural language inputs or (ii) only community metadata. In many real-world applications, collecting structured preference data -- such as numerical matrices or rankings -- can be difficult or impractical. We explore whether LLMs can accurately interpret natural language descriptions of preferences or infer preferences from community metadata (e.g., age, education, and location). Our results indicate that LLMs are capable of effectively processing nuanced, human-readable information, yielding improved decision-making compared to scenarios where they rely on explicit numerical representations. This suggests that LLMs have significant potential for mechanism design applications, particularly when traditional methods for collecting structured preferences are infeasible or costly~\citep{nisan2006bidding}.

\paragraph{Our Contributions.}
To summarize, we investigate whether LLMs can function as effective social planners in resource allocation tasks. This perspective also enables us to establish resource allocation as a functional benchmark for evaluating LLM reasoning. We examine how prompt design influences their decision-making behavior and demonstrate that LLMs can infer community preferences from natural language or metadata. Our key contributions are:

\begin{enumerate}[leftmargin=*,itemsep=0.1em]
    \item We introduce a mechanism design-based framework to evaluate LLM reasoning using real-world Participatory Budgeting (PB) instances from the \textsc{Pabulib} library~\cite{faliszewski2023participatory}. In this framework, the LLM acts as a social planner, allocating limited resources subject to feasibility constraints (refer to Figure~\ref{fig:teaser}).
    
    \item We study three prompt strategies: (i) \greedy, which imitates a utility-maximizing oracle~\citep{rey2023computational}; (ii) \optimization, which treats the PB task as a knapsack-style optimization problem~\cite{yamada2002heuristic}; and (iii) \hillclimb, where the LLM builds the final allocation by first reasoning over five smaller project subsets.
    
    \item We show that LLMs can perform competitively (sometimes even outperforming their structured-input variants) when provided only high-level community metadata (e.g., age, education, location). By prompting the model to output its inferred preferences, we assess its \textit{Theory-of-Mind} abilities~\citep{Kosinski2024,premack1978does}, by comparing inferred preferences to ground-truth data.
\end{enumerate}

%% file: sections/rw.tex
Recent research highlights the success of LLMs as agents. For instance, \citet{park23gen} discuss their potential for studying collective linguistic behavior. \citet{zhang-etal-2024-exploring} study NLP systems as LLM agents collaborating with human-like behavior and enhanced efficiency. We next review prior works, categorized into those focusing on LLMs as agents and those on their reasoning abilities.

\smallskip
\noindent\textbf{LLMs as Game-theoretic Agents.}
Recent works explore the intersection of game theory, mechanism design, and LLMs. \citet{duetting2024mechanism} examine auction-based mechanisms where LLMs participate in content creation, studying how bidding behavior influences AI outputs. Similarly, \citet{dubey2024auctions} analyze auctions where bidders compete for placement in LLM-generated summaries, such as ad auctions embedded in AI-generated text. \citet{zhu2024evidence} investigate LLM behavior in simulated auctions, introducing a synthetic data process to guide auction design.

\paragraph{LLM-based Agent Negotiation.}
% LLMs have also been studied as autonomous negotiators. 
CAMEL~\citep{li2023camel} introduces a role-playing framework where LLM agents collaborate via inception prompting, reducing human intervention in problem-solving. \citet{akata2023playing} examine LLM cooperation in two-player games, noting strong performance in simple strategies but suboptimal behavior in coordination games. \citet{mozikov2024good} show that emotional prompting (e.g., instructing GPT-4 to respond "angrily") disrupts negotiation alignment. \citet{bianchi2024llms} develop \textsc{NegotiationArena}, for benchmarking LLM negotiation strategies, showing that behavioral tactics improve outcomes. \citet{frisch-giulianelli-2024-llm} study language interaction in persona-conditioned LLM agents in negotiation, revealing varying personality consistency and linguistic alignment.

%%% https://aclanthology.org/2024.acl-long.331.pdf

\paragraph{LLMs as Agents in Urban Planning.}
% LLMs are increasingly applied in urban planning.
\citet{wang2024large} discuss LLM's role in spatial analysis and city planning. \citet{jiang2024urbanllm} fine-tune \texttt{LLaMa-2-7B} on spatio-temporal data from the Singapore Open Data API, improving autonomous urban activity planning. \citet{zhou2024large} introduce a multi-agent framework simulating participatory planning through discussions between urban planners and resident agents. \citet{sel2025llms} explore LLMs' ability to generate long-horizon plans, enhancing the Algorithm-of-Thoughts (AoT) framework for better planning benchmarks.

\paragraph{Benchmarking LLM Reasoning.}
Recent studies use negotiation as a benchmark for LLM reasoning. \citet{abdelnabi2024llmdeliberation} introduce a six-party, multi-issue negotiation to assess deliberation, while \citet{davidson2024evaluating} examine multi-agent negotiations with opposing and compatible interests to evaluate model competition and instruction-following. These works highlight negotiation as a proxy for complex reasoning and decision-making. 

\citet{wang-etal-2024-rethinking-bounds} show that single-agent LLMs with strong prompts match multi-agent discussions, which excel only without prompt demonstrations. \citet{heyman2025evaluatingsystematicreasoningabilities}, the closest work to ours, evaluate LLMs on \textit{graph-coloring}, a constrained optimization problem akin to PB, showing that LLMs consistently fail to respect constraints, even in simple coloring setups.

Unlike prior work on LLMs as (game-theoretic) agents or reasoning benchmarks, we examine their role as \textit{social planners} in participatory budgeting (PB), assessing their ability to allocate resources given \textit{real-world} instances. Our findings highlight LLMs' capacity to process/infer complex preferences and act as autonomous policymakers in dynamic, preference-driven environments.

%% file: sections/background.tex
Mechanism design~\citep{narahari2014game} is a field in game theory that focuses on structuring decision-making processes to achieve optimal outcomes, even when participants act in their self-interest. It primarily involves designing rules that incentivize truthful behavior and efficient allocations, providing a framework for evaluating decision-making systems. Participatory Budgeting (PB) fits naturally in this framework and offers a dynamic setting to test LLMs' reasoning capabilities.

\paragraph{PB Model}~\citep{aziz2021participatory,rey2023computational}. A standard PB model is denoted by the tuple $\textsc{PB} := \langle \mathbb{P}, \mathbb{V}, \beta, \texttt{cost} \rangle$. Here, $\mathbb{P}:=[m]$ denotes the set of $m\in \mathbb{N}$ \emph{projects}, $\mathbb{V}:=[n]$ denotes the set of $n\in \mathbb{N}$ \emph{voters}, and $\beta \in \mathbb{R}_{>0}$ is the available budget. The cost function  $\texttt{cost}:\mathbb{W}\rightarrow \mathbb{R}_{>0}$ defines the cost of provisioning a (sub)set of projects $\mathbb{W}\subseteq \mathbb{P}$. That is, $\texttt{cost}(\mathbb{W}):=\sum_{p\in \mathbb{W}} \texttt{cost}(p)$, where $\texttt{cost}(p)$ is the cost of provisioning the singleton set (or project) $p$.

\paragraph{Eliciting Preferences}~\citep{aziz2021participatory,rey2023computational}. Each voter $i \in \mathbb{V}$ expresses her preference for each project $j \in \mathbb{P}$ by revealing her score $s_i(p)$. The score can either be binary approval $s_i(p)\in \{0,1\}$, where the voter indicates her support for a project, or cardinal $s_i(p)\in \mathbb{R}_{\geq 0}$, where the voter provides a quantitative rating reflecting the degree of support for a project.

\paragraph{Utility Models.} \emph{Utility} helps quantify the value a voter receives from the provision of different projects. Given a set of provisioned projects $\mathbb{W}\subseteq \mathbb{P}$, for each voter $i \in \mathbb{V}$, there are two (natural) ways to define her utility~\citep{faliszewski2023participatory}.

First, \emph{score utility}, where the utility does not depend on the costs of the projects. That is, the score utility for each voter $i \in \mathbb{V}$ is $u_i^\texttt{sc}=\sum_{p \in \mathbb{W}} s_i(p)$.
Second, \emph{cost utility}, which assumes that costlier projects carry more value to the voters. That is, the cost utility for each voter $i \in \mathbb{V}$ is $u_i^\texttt{cost}=\sum_{p \in \mathbb{W}} s_i(p)\cdot \texttt{cost}(p)$.

When the distinction between score and cost utility is not necessary, we use $u_i$ to refer to voter $i$'s utility.
Utility is a fundamental metric that ensures selected projects maximize community benefit. To derive such a subset in PB, the social planner can employ algorithms such as the Utilitarian Greedy (UG) algorithm~\citep{rey2023computational}.

\paragraph{Utilitarian Greedy (UG) Algorithm.}
The UG algorithm~\citep{rey2023computational} aims to maximize the total utility of selected projects within a budget constraint. It starts with an empty set \( \mathbb{W} \). It iteratively adds a project \( p \) to $\mathbb{W}$ that maximizes the utility-to-cost ratio \( \sum_{i \in \mathbb{V}} \frac{u_i(p)}{\texttt{cost}(p)} \), provided that the total cost remains within the budget \( \beta \), i.e., 
    $\texttt{cost}(\mathbb{W}) + \texttt{cost}(p) \leq \beta. $

The process continues until no more projects can be added without exceeding the budget. Algorithm~\ref{algo::ug} provides the formal description for completeness. The UG algorithm is provably optimal up to one project~\citep{dantzig1957discrete}. That is, for any  allocation $\mathbb{W}$ returned by the UG algorithm, there exists another project $p \not\in \mathbb{W}$ such that~\citep{faliszewski2023participatory}:
$$
\sum_{i\in\mathbb{N}} u_i\left(\mathbb{W}\cup \{p\}\right) \geq \max_{\mathbb{W}^\prime:\texttt{cost}(\mathbb{W}^\prime)\leq \beta} \sum_{i\in\mathbb{N}} u_i\left(\mathbb{W}^\prime\right).
$$

\subsection{\textsc{Pabulib}: \underline{PA}rticipatory \underline{BU}dgeting \underline{LIB}rary}
\textsc{Pabulib}~\citep{faliszewski2023participatory} is an open-source PB library with real-world PB instances. 

\paragraph{PB Instance}~\citep{faliszewski2023participatory}. A typical PB instance is a single UTF-8 encoded text file with a "\texttt{.pb}" extension, divided into three sections. The \texttt{META} section includes general metadata, such as the country, budget, and the number of voters. The \texttt{PROJECTS} section details project costs, categories, and targets. The \texttt{VOTES} section records the votes, which can be of four types -- approval, ordinal, cumulative, or scoring -- and may also include voter metadata, such as age, education, and sex. Appendix~\ref{app::pabulib} presents a sample PB instance.

\paragraph{Dataset.}
\textsc{Pabulib} contains approximately 1,200 \texttt{.pb} instances collected from community projects in cities like Amsterdam, Krakow, and Warszawa, as well as through Amazon Mechanical Turk. We focus on 24 instances from Mechanical Turk (Appendix~\ref{app::pabulib}) as (i) algorithmic solutions cannot be computed for many of the 1,200 instances, (ii) several do not include voter metadata, and (iii) some are too large for an LLM's context window. These 24 instances are also in English -- unlike other \texttt{.pb} instances -- ensuring that performance comparisons across LLMs are not influenced by the models' potential biases toward English~\citep{wendler-etal-2024-llamas}. 

These selected 24 instances feature either 10 or 20 projects across categories such as Environment, Culture, Security, and Education. Each instance includes  $\approx 75$ voters, with voter metadata containing their age, sex, and education. Voters express their preferences as subsets of projects, and each instance operates under a fixed budget of USD 400K or 500K.

%% file: sections/methodology.tex
\subsection{Setting up the PB Instances}

To evaluate an LLM’s reasoning and instruction-following capabilities in participatory budgeting (PB), we consider three distinct types of PB instances. Each type presents a unique challenge, enabling us to assess the model’s ability to (i) process structured preferences (\ppi), (ii) infer missing preferences (\vrpi), or (iii) interpret preferences expressed in natural language (\nlvpi).
\begin{enumerate}[leftmargin=*,itemsep=0.1em]
    \item \textbf{Plain PB Instance} (\ppi): The standard PB setup (Section~\ref{sec::bg}). The LLM processes a structured PB instance, including project details, numerical voter preferences, and budget constraints, to determine a utility-maximizing $\mathbb{W}$.
    
    \item \textbf{Vote-Removed PB Instance} (\vrpi): Here, voter preferences are excluded. \vrpi\ comprises only project descriptions, budget constraints, and voter metadata (e.g., age, sex, education). \vrpi\ tests whether the LLM can infer likely preferences and optimize allocations using demographic and contextual cues.
    
    \item \textbf{Natural Language Votes PB Instance} (\nlvpi): Here, voter preferences are provided in natural language instead of numerical data to assess the LLM’s ability to interpret unstructured textual inputs. \nlvpi\ attempts to simulate real-world scenarios where preferences may be expressed informally rather than as explicit votes. We construct \nlvpi\ by prompting \texttt{gpt-4o}~\cite{openai2024gpt4o} to translate numerical votes from \ppi\ into natural language descriptions.
\end{enumerate}
Appendix~\ref{app::pabulib} provides an example instance of each of these types.  
By comparing an LLM's performance across these settings, we gain insights into its ability to process structured vs. unstructured data, predict missing preferences, and adapt to different preference formats.

\subsection{Few-shot Prompting \& LLMs in Focus}
For our experiments, we use few-shot prompting~\citep{brown2020language}. Specifically, we provide PB instances as exemplars to the LLM. Each exemplar includes (i) the PB instance (either \ppi, \vrpi, or \nlvpi) and (ii) a walkthrough of the UG algorithm applied to a dummy PB instance.

Given an LLM\footnote{We refer to LLMs generically without discussing their size, as size is not central to our focus.} \( f_\theta(\cdot) \) with parameters \( \theta \), the few-shot prompting set can be defined as:
$
\mathbb{C} := \{ I, x_1, \dots, x_k, y \}
$
where \( I \) is the \textit{system prompt} (explaining the PB setup), and \( x_1, \dots, x_k \) are the demonstration exemplars. Here, \( k \in \mathbb{Z}_{\geq 0} \) is the number of exemplars provided. With \( k = 0 \), this is zero-shot prompting, with \( k \geq 1 \), it is few-shot. The exemplars are dummy instances either of type \ppi, \vrpi, or \nlvpi\ and include their respective UG solutions. We use \( k = 2 \)\footnotemark for our experiments.

\footnotetext{Increasing \( k \) (e.g., \( k=4 \)) did not improve performance, likely because the simpler PB instances did not reveal additional patterns beyond the \texttt{.pb} structure and metadata. \( k \) is also limited by the LLM's context window.}

The test PB instance, \( y \), is one of the 24 real-world PB instances from \textsc{Pabulib}, which could be of type \texttt{PPI}, \texttt{VRPI}, or \texttt{NLVPI}. The LLM uses the setup and exemplars to generate an allocation for the test instance \( y \), i.e., the LLM outputs an allocation \( \mathbb{W} = f_\theta(y; x_1, \dots, x_k, I) \).

\smallskip
\noindent\textbf{LLMs in Focus.}
We evaluate
 auto-regressive LLMs, focusing on instruction fine-tuned versions of: (i) \llamaThreeSeventyB~\citep{dubey2024llama}, (ii) \llamanvidia~\citep{wang2024helpsteer2} (\texttt{NVIDIA}'s official fine-tuned variant), and (iii) \qwen\footnotemark~\citep{yang2024qwen2}. We prioritize open-source models for reproducibility.

\footnotetext{As of writing, \qwen\ and \llamaThreeSeventyB\ rank 1st and 3rd on the Hugging Face Open LLM leaderboard (among official providers)~\citep{openlmleaderboard}.}

\input{system_prompts/main_prompts}

\subsection{Prompt-based Constrained Optimization}

Across all prompt variants, we instruct the LLM to act as a social planner in a participatory budgeting (PB) task. The goal is to select a subset of projects that maximizes total social welfare while strictly adhering to a budget constraint. The variants differ in the strategies they employ for project selection, as detailed below.

\begin{enumerate}[leftmargin=*, itemsep=0.25em]
    \item {\greedy.} The prompt explicitly directs the LLM to follow the Utilitarian Greedy (UG) algorithm verbatim. In this approach, the LLM iteratively selects the project with the highest utility-to-cost ratio, as long as it fits within the budget. Figure~\ref{fig:greedy_strategy} illustrates this prompt style.
    
    \item {\optimization.} The prompt guides the LLM to treat the task as a combinatorial optimization problem, similar to the \texttt{Knapsack} problem~\cite{yamada2002heuristic}. This strategy emphasizes calculating utility-to-cost ratios and selecting the highest feasible projects in a way that respects the budget constraint. Figure~\ref{fig:knapsack_strategy} shows this setup.
    
    \item {\hillclimb.} The prompt instructs the LLM to approach project selection as a local search procedure, inspired by {Hill Climbing}~\cite{peter2021modern}. The LLM incrementally builds an allocation by greedily adding the most utility-efficient project that fits within the remaining budget. Figure~\ref{fig:hill_climbing_strategy} demonstrates this approach.
\end{enumerate}

In all variants, the prompt requires the LLM to provide a detailed chain of thought. The model must explain the strategies used to generate subsets, evaluate how each project is assessed for selection or exclusion, and track the running total of utility and cost at each step.

\subsection{Performance and Reasoning Metrics}
\noindent\textbf{Performance Measure.} We focus on \emph{average utility}, (\texttt{AU}), a natural metric of efficiency. It is defined as $\sum_{i\in\mathbb{N}} u_i(\mathbb{W})$, where $\mathbb{W}\subseteq \mathbb{P}$ is the subset of the projects provisioned/allocated. \texttt{AU} measures how much benefit or satisfaction a community gains from the allocated projects $\mathbb{W}$. We report the \textit{normalized} \texttt{AU}, i.e., the ratio of the \texttt{AU} from an LLM's allocation with the \texttt{AU} from the UG algorithm.

\smallskip
\noindent\textbf{Reasoning Measure.} The metric of interest here is: \textit{instruction-following} (IF)~\citep{davidson2024evaluating}. IF is essential for the secure deployment of LLM agents and their ability to perform tasks efficiently. We measure \texttt{IF} based on the LLM's output, specifically, the LLM's ability to (i) format its output based on the instruction provided (formats may also change based on the PB type) and (ii) respect the PB game constraints (e.g., the LLM's allocation $\mathbb{W}$ must be within the budget $\beta$).

%% file: system_prompts/main_prompts.tex
\begin{figure}[t]
\centering
\footnotesize
\noindent\fbox{%
\parbox{0.95\linewidth}{%
\begin{itemize}[itemsep=0.1em]
    \item \textbf{Objective:} Select a subset of projects that maximizes total utility (voter approvals) while strictly respecting a budget constraint.
    \item \textbf{Greedy Strategy:} Implement the Utilitarian Greedy (UG) algorithm: compute each project's utility-to-cost ratio and iteratively select the project with the highest ratio, as long as the budget allows.
    \item \textbf{Step-by-Step Reasoning:} For each project considered:
    \begin{itemize}[noitemsep]
        \item Calculate the utility-to-cost ratio,
        \item Check whether it fits within the remaining budget,
        \item Include it if feasible, otherwise skip it.
    \end{itemize}
    \item \textbf{Full Pass Required:} Sort all projects by utility-to-cost ratio and consider them in that order, stopping only when no more projects can be added without exceeding the budget.
    \item \textbf{Output:} Return a JSON object formatted as follows:
\end{itemize}

\begin{center}
\texttt{
\{ "Chain-of-Thought": "",
"allocation": [] \}
}
\end{center}

}%
}
\caption{\greedy: Mimicking UG Algorithm~\cite{faliszewski2023participatory} for Project Selection}
\label{fig:greedy_strategy}
\end{figure}%

\begin{figure}[t]
\centering
\footnotesize
\noindent\fbox{%
\parbox{0.95\linewidth}{%
\begin{itemize}[itemsep=0.1em]
    \item \textbf{Objective:} Select a subset of projects that maximizes total utility (voter approvals) while strictly respecting a budget constraint.
    \item \textbf{Knapsack-Style Strategy:} Formulate the problem as a 0-1 knapsack problem: compute each project's utility-to-cost ratio, and iteratively select the highest feasible project, avoiding those that exceed the budget.
    \item \textbf{Step-by-Step Reasoning:} For each project, explain the following:
    \begin{itemize}[noitemsep]
        \item The calculated utility-to-cost ratio,
        \item Whether the project is selected or skipped,
        \item The result of the budget check (i.e., feasibility).
    \end{itemize}
    \item \textbf{Full Pass Required:} Consider all projects in order of descending utility-to-cost ratio and stop only when no further additions are feasible.
    \item \textbf{Output:} Return a JSON object formatted as follows:
\end{itemize}

\begin{center}
\texttt{
\{ "Chain-of-Thought": "",
"allocation": [] \}
}
\end{center}

}%
}
\caption{\optimization: Knapsack-Style Strategy for Project Selection}
\label{fig:knapsack_strategy}
\end{figure}

\begin{figure}[t]
\centering
\footnotesize
\noindent\fbox{%
\parbox{0.95\linewidth}{%
\begin{itemize}[itemsep=0.1em]
    \item \textbf{Objective:} Maximize total utility (voter approvals) under a strict budget constraint in a participatory budgeting setting.
    \item \textbf{Hill-Climbing Strategy:} Frame the final allocation step as a hill-climbing process. Iteratively select projects from a candidate pool that most improve the utility, guided by utility-to-cost ratio or similar heuristics. Stop when no feasible addition improves the solution within budget.
    \item \textbf{Step-wise Instructions:}
    \begin{itemize}[noitemsep]
        \item Generate 5 small, diverse subsets of projects (3–5 projects each) using distinct strategies (high utility, low cost, high U/C ratio, thematic diversity, balance).
        \item Merge all subsets into a deduplicated candidate pool.
        \item Apply a greedy, hill-climbing-like heuristic to construct the final allocation:
        \begin{itemize}[noitemsep]
            \item Sort projects by utility-to-cost ratio.
            \item Add projects one by one, updating total utility and cost.
            \item Skip any project that would exceed the budget.
        \end{itemize}
    \end{itemize}

    \item \textbf{Output:} Return a JSON object formatted as follows:
\end{itemize}

\begin{center}
\texttt{
\{
"5\_individual\_allocations": [],
"Chain-of-Thought": "",
"allocation": []
\}
}
\end{center}

}%
}
\caption{\hillclimb: Tailored "Hill-Climbing" Strategy for Project Selection}
\label{fig:hill_climbing_strategy}
\end{figure}

%% file: sections/results.tex
We now present the results for the performance and reasoning metrics outlined in Section~\ref{sec::method} for the LLMs in focus. For each PB type, \ppi\, \vrpi\, or \nlvpi\, we benchmark the LLMs on the three prompt strategies: \greedy, \optimization, and \hillclimb. For each metric, we report the mean and standard deviation across three independent runs. 
Appendix~\ref{app::more_exps} provides other details, including hyperparameters, system prompts, and \textit{sample outputs}. The code base is available at: \url{github.com/sankarshandamle/llm-pb}.

%%%%%%%%%%
%%%%%%%%%%
%%%%%%%%%%
%%%%%%%%%%
\input{tables/main_result_if}

\input{tables/main_result_au}
%%%%%%%%%%
%%%%%%%%%%
%%%%%%%%%%
%%%%%%%%%%

\smallskip
\noindent {\texttt{IF} \& \textbf{Normalized} \texttt{AU}.}
Table~\ref{tab:main_result_if} and Table~\ref{tab:main_result_au} present the Instruction-Following (\texttt{IF}) and normalized Average Utility (\texttt{AU}) performance across models and PB setups under three methods. Across the board, \qwen\ consistently exhibits strong performance. For instance, under the \greedy\ method, it achieves the highest \texttt{IF} on \vrpi\ ({0.94 $\pm$ 0.052}) and a top \texttt{AU} on \nlvpi\ ({0.655 $\pm$ 0.043}). Notably, \llamanvidia\ under \greedy\ performs best on \ppi\ with an \texttt{IF} of {0.89 $\pm$ 0.039}, while \llamaThreeSeventyB\ leads on \vrpi for \texttt{AU} with {0.573 $\pm$ 0.058}.

When equipped with optimization routines, all models see improvements in both metrics. \qwen\ under \optimization\ outperforms others with \texttt{AU} scores above 0.78 for all setups, peaking at {0.896 $\pm$ 0.021} on \ppi\. Meanwhile, \llamanvidia\ shows an \texttt{IF} above 0.94 for both \ppi\ and \nlvpi\, but under performs in \texttt{AU} on \ppi\ ({0.505 $\pm$ 0.000}).

In most cases, \hillclimb\ yields further improvements. E.g., \llamaThreeSeventyB\ achieves an \texttt{AU} of {0.842 $\pm$ 0.000} on \nlvpi, the highest among its configurations. \qwen\ also performs well with strong \texttt{IF} on \vrpi\ ({0.875 $\pm$ 0.090}) and high utility scores ({0.876 $\pm$ 0.009}) on \ppi. In summary, \qwen\ consistently offers the best trade-off between instruction-following and utility across methods. Introducing \optimization\ and \hillclimb\ significantly boosts performance, especially for weaker base models like \llamaThreeSeventyB. These findings highlight the role of inference-time reasoning in improving both the feasibility and quality of LLM-generated allocations.

\subsection{Results Takeaways}
Table~\ref{tab:main_result_if} and Table~\ref{tab:main_result_au} highlight an interesting insight into the performance of LLMs when tasked with participatory budgeting (PB). Notably, setups involving (i) natural language preferences (\nlvpi) and (ii) scenarios where the LLM infers voter preferences (\vrpi) often exhibit higher average utility (\texttt{AU}) compared to the setup where preferences are presented as numerical matrices (\ppi). For instance, Table~\ref{tab:main_result_au}, we see that the average performance for \ppi\ and \vrpi\ is similar, while it is marginally higher for \nlvpi.

In the \nlvpi\ setting, \qwen\ achieves a normalized utility of $0.655$ under the \greedy\ method and up to $0.789$ with \optimization, outperforming its corresponding results in the \ppi\ setting ($0.590$ and $0.896$, respectively). A similar trend is observed for \llamanvidia, which improves from $0.554$ (\ppi) to $0.620$ (\vrpi) under \greedy, and from $0.505$ to $0.656$ under \optimization. These patterns indicate that the model is able to reason better when dealing with richer, inferred voter preferences -- such as those derived from metadata -- than with sparse project-level inputs. Overall, the average normalized utility across all models is highest in \nlvpi\ ($0.695$), compared to \vrpi\ ($0.676$) and \ppi\ ($0.673$), highlighting the importance of fine-grained preference signals in improving LLM-based allocations.

We believe these findings are significant and suggest that LLMs are not only adept at processing natural language but may also outperform when interpreting or inferring preferences from textual inputs compared to structured preference matrices. The results indicate the potential for deploying LLMs as social planners in mechanism design, particularly in settings where explicit preferences are unavailable or challenging to represent numerically (e.g., combinatorial preferences~\citep{nisan2006bidding}).
% The ability of LLMs to process and infer preferences from textual data demonstrates their potential for solving complex allocation problems in real-world participatory decision-making, further bridging the gap between computational techniques and social applications.

%%%%%%%%%%%%%%%%%%%%%%%%%%%%%%%
%%%%%%%%%%%%%%%%%%%%%%%%%%%%%%%
%%%%%%%%%%%%%%%%%%%%%%%%%%%%%%%
\begin{figure}[t]
    \centering
        \centering
        \fbox{%
        \tiny
        \begin{minipage}[t]{0.95\columnwidth}
        \ttfamily
       \texttt{\hl{`allocation': ['5', '3', '1', '22', '12']}},
        \texttt{`Chain-of-thought': ``I chose projects 5, 3, 1, 22, and 12 because they have the highest number of votes. Project 5 has the most votes with 50, project 3 has 37 votes, project 1 has 10 votes, project 22 has 17 votes, and project 12 has 15 votes. The total cost of these projects is 239,676 + 128,805 + 77,344 + 320,000 + 50,000 = 815,825, which exceeds the budget of 500,000. To meet the budget, I initially selected the projects with the highest votes and adjusted the selection by choosing only those that fit within the total budget. \hl{Therefore, I ended up choosing projects 5, 3, and 1."}}
        \end{minipage}
        }
    \caption{\textbf{Internal (but not External) Faithfulness:} A sample output of \llamaThreeSeventyB, for \greedy\ with \ppi. While the reasoning traces a budget-adjustment process, the final allocation does not reflect the stated revision, indicating a mismatch between explanation and output.}
    \label{fig:model_responses}
\end{figure}

\paragraph{Faithfulness~\citep{jacovi2020towards}.} Faithfulness measures how accurately an LLM's reasoning explains its answers. Recall that we require LLMs to (i) produce a final allocation $\mathbb{W}$ and (ii) provide reasoning for selecting $\mathbb{W}$. In this context, faithfulness refers to situations where the LLM's allocation $\mathbb{W}$ logically follows from its reasoning. An LLM is considered \textit{internally but not externally faithful} if its final allocation $\mathbb{W}$ contradicts its reasoning~\citep{davidson2024evaluating}. 

Figure~\ref{fig:model_responses} depicts an instance of \llamaThreeSeventyB's output, which shows internally but not externally faithful behavior. While we acknowledge such nuanced reasoning behavior, our \texttt{IF} metric still marks this instance as a failure -- because in real-world deployment scenarios, only the final allocation $\mathbb{W}$ would be actioned, and if it violates the budget constraint, it cannot be accepted regardless of the LLM's internal response.

%%%%%%%%%%%%%%%%%%%%%%%%%%%%%%%
%%%%%%%%%%%%%%%%%%%%%%%%%%%%%%%
%%%%%%%%%%%%%%%%%%%%%%%%%%%%%%%
\begin{figure}[t]
    \centering
        \fbox{%
        \begin{minipage}{0.95\columnwidth}
        \tiny
        \ttfamily
        \texttt{\{'allocation': [31, 2, 23, 13, 14, 3], 'Chain-of-Thought': "  Projects: - ID: 3, Cost: 27,000, Votes: 48 - ID: 14, Cost:\
      \ 25,000, Votes: 41 - ID: 13, Cost: 24,000, Votes: 40 - ID: 31, Cost: 8,000,\
      \ Votes: 38 - ID: 2, Cost: 13,000, Votes: 39 - ID: 23, Cost: 40,000, Votes:\
      \ 36  Steps: 1. Calculate the utility-to-cost ratio for each project. 2.\
      \ Sort projects by this ratio in descending order. 3. Select projects one by\
      \ one, checking the budget constraint.   Output: 1. Instance: Knapsack\_7 \
      2. Selected Projects: [31, 2, 23, 13, 14, 3] 3. Explanation:    - Step 1:\
      \ Calculate ratios:      - Project 31: Utility = 38, Cost = 8,000, Ratio =\
      \ 38 / 8,000 \u2248 0.00475      - Project 2: Utility = 39, Cost = 13,000,\
      \ Ratio = 39 / 13,000 \u2248 0.003      - Project 23: Utility = 36, Cost =\
      \ 40,000, Ratio = 36 / 40,000 \u2248 0.0009      - Project 13: Utility = 40,\
      \ Cost = 24,000, Ratio = 40 / 24,000 \u2248 0.00167      - Project 14: Utility\
      \ = 41, Cost = 25,000, Ratio = 41 / 25,000 \u2248 0.00164      - Project 3:\
      \ Utility = 48, Cost = 27,000, Ratio = 48 / 27,000 \u2248 0.00178      - Others\
      \ have lower ratios or zero utility.    - Step 2: Sort: [31, 2, 23, 13, 14,\
      \ 3,...]    - Step 3: Add Project 31 (Cost = 8,000, Remaining Budget = 500,000\
      \ - 8,000 = 492,000).    - Step 4: Add Project 2 (Cost = 13,000, Remaining\
      \ Budget = 492,000 - 13,000 = 479,000).    - Step 5: Add Project 23 (Cost =\
      \ 40,000, Remaining Budget = 479,000 - 40,000 = 439,000).    - Step 6: Add\
      \ Project 13 (Cost = 24,000, Remaining Budget = 439,000 - 24,000 = 415,000). \
      \   - Step 7: Add Project 14 (Cost = 25,000, \hl{Remaining Budget = 415,000 - 25,000
      \ = 390,000).    - Step 8: Add Project 3 (Cost = 27,000, Remaining Budget =
      \ 390,000 - 27,000 = 363,000).    - Step 9: Next project exceeds the budget.}\
      \ Stop. 4. Return selected projects ( W = [31, 2, 23, 13, 14, 3] ).  "}
        \end{minipage}
        }
 \caption{Output of \llamaThreeSeventyB\ when presented with a PB instance of type \ppi\ with \greedy\ as the prompt strategy. The model stops prematurely, with the baseline UG allocation consuming USD 481,400 out of the USD 500,000 budget, while the model's allocation only consumes USD 137,000. Additionally, the model "loosely" follows the greedy allocation order. The actual greedy order is $[3, 14, 13, 12, 2, 31, 21, 23, 33, 41, 15, 50, 42, 43, 47, 22, 36, 5, 37, 8]$.}
    \label{fig:output_greedy_base_main}
\end{figure}

%%%%%%%%%%%%%%%%%%%%%%%%%%%%%%%
%%%%%%%%%%%%%%%%%%%%%%%%%%%%%%%
%%%%%%%%%%%%%%%%%%%%%%%%%%%%%%%

\smallskip
\noindent\greedy\ \textbf{vs.} \optimization\ \textbf{\&} \hillclimb. We attribute the relatively poorer performance of \greedy\ to the following factors: Analysis of the \texttt{Chain-of-Thought} field reveals that LLMs often execute the \texttt{UG} algorithm inconsistently and conservatively. Rather than strictly following the utility-per-cost ordering, models tend to apply a "loose" greedy strategy, omitting or reordering projects in ways that deviate from the expected behavior (refer to Figure~\ref{fig:output_greedy_base_main} for one such instance). Additionally, LLMs frequently terminate the allocation process early, leaving a significant portion of the budget unutilized. We present other model outputs in Appendix~\ref{app::sample_outputs}.

\begin{figure}
    \centering
    \includegraphics[width=0.95\linewidth]{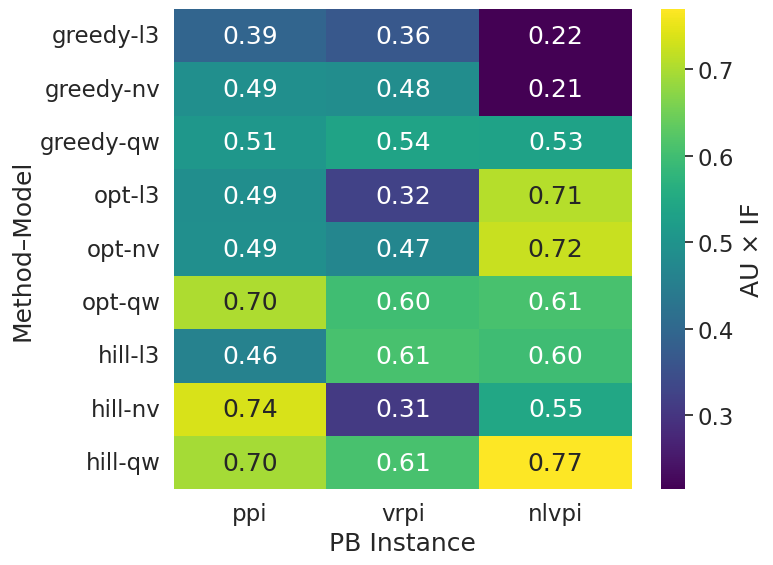}
    \caption{Illustrating the tradeoff between Instruction-following (\texttt{IF}) and Average Utility (\texttt{AU}) for various methods across three different PB instances: \ppi, \vrpi, and \nlvpi. Each cell represents the normalized value of \texttt{AU} $\times$ \texttt{IF}, highlighting the interplay between these two metrics. }
    \label{fig:trade-off}
\end{figure}

\paragraph{IF-AU Trade-off.} 
Figure~\ref{fig:trade-off} visualizes the interaction between \texttt{AU} and \texttt{IF}, using their product to quantify trade-offs between utility and consistency across LLMs. \qwen, under both \optimization\ and \hillclimb, achieves high trade-off values across PB instances -- particularly in \nlvpi\ (0.72 and 0.77) -- indicating robust alignment between utility maximization and instruction consistency. In contrast, \llamaThreeSeventyB, especially under \greedy\ and \optimization, exhibits weaker trade-offs (e.g., 0.36 and 0.22 for \texttt{greedy-l3}; 0.32 and 0.71 for \texttt{opt-l3}), highlighting less stable performance. These results suggest that optimization-based prompting paired with stronger LLMs yields more reliable LLM behavior.

\subsubsection{Ablation Study} 

We also present an ablation based on \textit{model size} (Appendix~\ref{app::subsec_add_exps}), examining two categories: (i) 7B–14B models (\llamaEight, \qwenOneFour, and \mistralSeven), and (ii) 3B models (\llamaThree\ and \qwenThree). We fix the method to \greedy, as these smaller models may find it easier to mimic the \texttt{UG} algorithm compared to executing the more complex reasoning required by \optimization\ or \hillclimb.

Generally, smaller models underperform their larger counterparts on both \texttt{AU} and \texttt{IF}. For instance, from Tables~\ref{tab:main_result_if_app} and~\ref{tab:main_result_au_app}, \llamaThree\ suffers a sharp drop of 0.321 in normalized utility (\texttt{AU}) and 0.329 in consistency (\texttt{IF}) relative to \llamaEight. Similarly, \qwenThree\ retains its \texttt{AU} score compared to \qwenOneFour, but its \texttt{IF} score declines by 0.152.

In the 7B–14B category, \mistralSeven\ achieves strong consistency on \ppi\ (0.833) but underperforms on more complex setups like \nlvpi\ (0.526). Across PB setups, \qwenOneFour\ emerges as the most consistent (avg. \texttt{IF} of 0.687), while \llamaEight\ leads on normalized utility in the \nlvpi\ setting (0.531). Notably, \texttt{AU} scores for the \ppi\ PB type are higher than for the others -- suggesting that reasoning based on structured input is easier for LLMs to interpret compared to PB types that require the LLM to predict preferences, or interpret from language-based encodings.

\subsection{Theory of Mind (ToM) Inference}
\citet{davidson2024evaluating} demonstrate that LLMs engaged in negotiation tasks, such as rent negotiations between a landlord and tenant, are capable of predicting the payoff structure of the opposing party, highlighting their ability to infer strategic preferences. This ability is referred to as \emph{Theory of Mind} (ToM) inference~\citep{Kosinski2024,premack1978does}. With the performance of LLMs for \vrpi, we see similar ability -- specifically, LLM's ability to infer voter preferences, given the project categories and the voter's metadata.

\smallskip\noindent\textbf{Fraction Overlap.} To see how effective an LLM is in correctly predicting voter preferences, we examine the overlap between the projects predicted by each LLM for individual voters and the \textit{ground truth} from the corresponding \ppi\ instances. We focus on the intersection of the top-75\% of the projects that receive the highest votes, as this metric helps assess whether the LLM accurately predicts the highest-voted projects. Appendix~\ref{app::subsec_add_exps} presents numbers for top-25\% and top-50\% of the projects.

\begin{table}[ht]
\centering
\caption{Normalized top-75\% Fraction Overlap, per method, for each LLM. We also report the Normalized Utility (AU) from Table~\ref{tab:main_result_au}.}
\label{tab:overlap_au}
\begin{adjustbox}{max width=\columnwidth}
\begin{tabular}{lccc}
\toprule
\textbf{Model} &   \textbf{Method}  & \textbf{Overlap} & \textbf{AU} \\
\midrule
\multirow{3}{*}{\llamaThreeSeventyB} &  \greedy & 0.621 & 0.573 \\
& \optimization  & 0.687 & 0.773 \\
& \hillclimb & 0.683 &  0.694 \\  
\midrule
\multirow{3}{*}{\llamanvidia }   & \greedy & 0.710 & 0.620 \\
& \optimization & 0.731 &  0.656 \\
& \hillclimb & 0.726 &  0.704 \\
\midrule
\multirow{3}{*}{\qwen }              & \greedy & 0.487 & 0.574 \\
& \optimization & 0.653 & 0.820 \\
& \hillclimb & 0.610 & 0.694 \\
\bottomrule
\end{tabular}
\end{adjustbox}
\end{table}

We compute the average top-75\% fraction overlap using the three original runs per setting to evaluate how well each LLM can identify high-utility projects based on inferred preferences from metadata. Under the \greedy\ method, \llamaThreeSeventyB\ achieves an overlap of 0.621 with normalized utility (AU) of 0.573, while \llamanvidia\ attains 0.710 (AU: 0.620), and \qwen\ 0.487 (AU: 0.574). Both \llamaThreeSeventyB\ and \llamanvidia\ exhibit notably higher overlap than \qwen, aligning with their stronger performance in the \vrpi\ setting over \ppi. This suggests that LLMs infer and reason over community preferences from metadata to produce allocations.

\subsection{Adapting LLM Evaluation with Project Conflicts}

To demonstrate the adaptability of our PB framework, we extend the setup by introducing an additional constraint -- \textit{project conflicts} -- alongside the existing budget constraint. By incorporating such constraints, we progressively increase the framework's complexity, enabling us to assess whether LLMs can generalize beyond simple numeric feasibility and adapt to more structured, nuanced allocation scenarios. This highlights the flexibility of our framework in accommodating evolving evaluation goals and reflecting the increasing complexity of real-world decision-making environments.

\subsubsection{Adding Project Conflicts\label{subsec::conflict}}
To demonstrate the potential of using LLMs as social planners, we extend the PB setting introduced in Section~\ref{sec::bg} by incorporating conflicts. Conflicting projects naturally arise in PB scenarios. For instance, two projects -- such as building a library and constructing a sports complex -- might conflict if they compete for the same land or resources, making it impossible to fund both simultaneously.

Formally, we define a set of projects $\mathbb{C} \subseteq \mathbb{P}$ as \textit{conflicting}, such that the final allocation $\mathbb{W}$ cannot include any pair of projects $a, b \in \mathbb{W}$ where $a$ and $b$ are in conflict, i.e., $a, b \in \mathbb{C}$. We also tweak the UG algorithm (Section~\ref{sec::bg} or Appendix~\ref{algo::ug}) so that it greedily selects projects that (i) respect the budget constraints and (ii) are not conflicting. Greedily, for each project $p$, we add $p$ to $\mathbb{W}$ if,
\begin{equation}\label{eqn::conflict}
    \begin{rcases}
        \texttt{cost}(\mathbb{W}) + \texttt{cost}(p) &\leq \beta \\
        \forall w \in \mathbb{W} \text{~s.t.~} p \not= w, \{p,w\} &\not\subseteq \mathbb{C}
    \end{rcases}
\end{equation}

We call this variant -- Algorithm~\ref{algo::ug} with Eq.~\ref{eqn::conflict} -- as the \emph{Modified} UG algorithm\footnotemark.

\footnotetext{We note that PB with project conflicts is analogous to the well-known Disjunctively Constrained Knapsack Problem (DCKP)~\citep{yamada2002heuristic}. The conflict set $\mathbb{C}$ can be naturally represented as the set of conflicting item pairs in DCKP. The Modified UG algorithm (Eq.~\ref{eqn::conflict}) also corresponds to the "GREEDY" algorithm discussed by \citet{yamada2002heuristic}.}

%%%%%%%%%%%%%%%%%%%%%%%%%%%%%%%
%%%%%%%%%%%%%%%%%%%%%%%%%%%%%%%
%%%%%%%%%%%%%%%%%%%%%%%%%%%%%%%

%%%%%%%%%%%%%%%%%%%%%%%%
%%%%%%%%%%%%%%%%%%%%%%%%%%%%%%%
%%%%%%%%%%%%%%%%%%%%%%%%%%%%%%%
\input{tables/project_conflict}
%%%%%%%%%%%%%%%%%%%%%%%%%%%%%%%
%%%%%%%%%%%%%%%%%%%%%%%%%%%%%%%
%%%%%%%%%%%%%%%%%%%%%%%%%%%%%%%

\paragraph{The Modified UG Algorithm's Social Welfare may not be Bounded.}  In our previous setup, the \texttt{UG} algorithm served as a challenging benchmark for LLMs to surpass, as it is optimal up to one project~\citep{dantzig1957discrete}. However, the introduction of project conflicts negatively impacts the quality of allocations produced by the greedy approach. Specifically, we observe that the social welfare under the greedy allocation (i.e., the sum of utilities, $\sum_{i\in\mathbb{V}}u_i(\mathbb{W_{\texttt{greedy}}})$) may be unbounded relative to the optimal social welfare, $\sum_{i\in\mathbb{V}}u_i(\mathbb{W_\texttt{opt}})$ (see Proposition 1 in Appendix~\ref{app::proofs}). Additionally, unlike the UG algorithm, the Modified UG algorithm is \emph{not} optimal up to one project (refer to Proposition 2 in Appendix~\ref{app::proofs}). The proofs are straightforward and follow by construction, and we defer the details to Appendix~\ref{app::proofs}.

\subsubsection{Setup and Results}

\paragraph{Setup.} To evaluate LLM performance in participatory budgeting (PB) with project conflict constraints, we augment each of the 24 PB instances by introducing a randomly sampled conflict set $\mathbb{C}$. For each project $p \in \mathbb{P}$, we include it in $\mathbb{C}$ with probability $q = 0.5$ and exclude it otherwise. The baseline average utility (\texttt{AU}) is computed using a modified version of the Utilitarian Greedy (UG) algorithm, which incorporates constraint checks via Eq.~\ref{eqn::conflict} to ensure feasibility. We use the \greedy\ method for comparison, as it allows us to isolate the impact of benchmark adaptation (e.g., introducing project conflicts) on LLM performance. This setup highlights how such adaptations influence allocation quality in a more straightforward manner. Moreover, \greedy\ is suitable in this context since, unlike the previous approach, the Modified \texttt{UG} algorithm is a less challenging algorithmic baseline for the LLM to surpass (refer to Section~\ref{subsec::conflict}).

\paragraph{Results.} Table~\ref{tab:main_result_conflict} presents the normalized \texttt{AU} and instruction-following (\texttt{IF}) scores for the three open-source LLMs. Here, \texttt{IF} includes (i) correct output format, (ii) budget feasibility, and (iii) non-conflicting allocation. That is, \texttt{IF} is strictly tougher for LLMs to respect here than the previous setup. Compared to the results in Table~\ref{tab:main_result_au}, the LLMs' average utility is now closer to the baseline. Figure~\ref{fig:project_conflict_instances} highlights a few instances where the LLM outperforms the baseline allocation. Notably, \qwen\ achieves the highest performance, with normalized \texttt{AU} and \texttt{IF} scores of 0.865 and 0.778, respectively -- consistent with its top-rank in recent evaluations~\citep{openlmleaderboard}.

Our results indicate LLMs' potential as effective decision-makers in complex, preference-driven environments where traditional mechanism design may struggle to represent or collect preferences.

%% file: tables/main_result_if.tex
\begin{table}[t]
\caption{\textbf{Instruction-following (\texttt{IF}) Comparison:} The fraction of PB instances for which the LLM produces a consistent allocation.}
\label{tab:main_result_if}
    \begin{center}
        \begin{adjustbox}{max width=\linewidth}
            \begin{tabular}{ccccc}
            \toprule
       \multirow{1}{*}{\textbf{Method}} &    \multirow{1}{*}{\bf Model}  &  \multirow{1}{*}{\ppi} &  \multirow{1}{*}{\vrpi} &  \multirow{1}{*}{\nlvpi} \\
            \midrule
     \multirow{3}{*}{\greedy}   
        & {\llamaThreeSeventyB}  & 0.69\textsubscript{\scriptsize{0.20}} & 0.64\textsubscript{\scriptsize{0.07}} & 0.47\textsubscript{\scriptsize{0.22}} \\
        & {\llamanvidia}  & 0.89\textsubscript{\scriptsize{0.04}} & 0.78\textsubscript{\scriptsize{0.02}} & 0.42\textsubscript{\scriptsize{0.00}} \\
        & \qwen  & 0.86\textsubscript{\scriptsize{0.00}} & 0.94\textsubscript{\scriptsize{0.05}} & 0.81\textsubscript{\scriptsize{0.03}} \\
    \midrule
    \multirow{3}{*}{\optimization} 
        & {\llamaThreeSeventyB} & 0.75\textsubscript{\scriptsize{0.03}} & 0.42\textsubscript{\scriptsize{0.09}} & 0.92\textsubscript{\scriptsize{0.03}} \\
        & \llamanvidia & 0.96\textsubscript{\scriptsize{0.00}} & 0.71\textsubscript{\scriptsize{0.00}} & 0.94\textsubscript{\scriptsize{0.00}} \\
        & \qwen & 0.78\textsubscript{\scriptsize{0.03}} & 0.73\textsubscript{\scriptsize{0.02}} & 0.77\textsubscript{\scriptsize{0.02}} \\
    \midrule
    \multirow{3}{*}{\hillclimb} 
        & {\llamaThreeSeventyB} & 0.64\textsubscript{\scriptsize{0.08}} & 0.88\textsubscript{\scriptsize{0.03}} & 0.71\textsubscript{\scriptsize{0.01}} \\
        & \llamanvidia & 0.97\textsubscript{\scriptsize{0.02}} & 0.44\textsubscript{\scriptsize{0.04}} & 0.78\textsubscript{\scriptsize{0.00}} \\
        & \qwen  & 0.79\textsubscript{\scriptsize{0.06}} & 0.88\textsubscript{\scriptsize{0.09}} & 1.00\textsubscript{\scriptsize{0.00}} \\
    \midrule
    \multicolumn{2}{c}{\textbf{Average}} 
        & 0.80 & 0.72 & 0.75 \\
            \bottomrule
            \end{tabular}
        \end{adjustbox}
    \end{center}
\end{table}

%% file: tables/main_result_au.tex
\begin{table}[t]
\caption{\textbf{Normalized \texttt{AU} Comparison:} The average utility for each PB instance (normalized by the \texttt{UG} allocation) produced by the LLM.}
\label{tab:main_result_au}
    \begin{center}
        \begin{adjustbox}{max width=\linewidth}
            \begin{tabular}{ccccc}
            \toprule 
   \multirow{1}{*}{\bf Method} &         \multirow{1}{*}{\bf Model}  &  \multirow{1}{*}{\ppi} &  \multirow{1}{*}{\vrpi} &  \multirow{1}{*}{\nlvpi}  \\
            \midrule
        \multirow{3}{*}{\greedy}  
        &  {\llamaThreeSeventyB}  & 0.57\textsubscript{\scriptsize{0.05}} & {0.57\textsubscript{\scriptsize{0.06}}} & 0.46\textsubscript{\scriptsize{0.10}}   \\
        & {\llamanvidia}  & 0.55\textsubscript{\scriptsize{0.01}} & {0.62\textsubscript{\scriptsize{0.01}}} & 0.51\textsubscript{\scriptsize{0.00}}   \\
        & \qwen  & 0.59\textsubscript{\scriptsize{0.01}} & 0.57\textsubscript{\scriptsize{0.02}} & {0.66\textsubscript{\scriptsize{0.04}}}  \\
        \midrule
    \multirow{3}{*}{\optimization} 
        & {\llamaThreeSeventyB} & 0.65\textsubscript{\scriptsize{0.01}} & 0.77\textsubscript{\scriptsize{0.03}} & 0.77\textsubscript{\scriptsize{0.03}}  \\
        & \llamanvidia &  0.51\textsubscript{\scriptsize{0.00}} & 0.66\textsubscript{\scriptsize{0.00}} & 0.77\textsubscript{\scriptsize{0.00}} \\
        & \qwen & 0.90\textsubscript{\scriptsize{0.02}} & 0.82\textsubscript{\scriptsize{0.02}} & 0.79\textsubscript{\scriptsize{0.00}} \\
        \midrule
    \multirow{3}{*}{\hillclimb} 
        & {\llamaThreeSeventyB} & 0.72\textsubscript{\scriptsize{0.06}} & 0.69\textsubscript{\scriptsize{0.03}} & 0.84\textsubscript{\scriptsize{0.00}}  \\
        & \llamanvidia & 0.76\textsubscript{\scriptsize{0.02}} & 0.70\textsubscript{\scriptsize{0.00}} & 0.70\textsubscript{\scriptsize{0.00}}  \\
        & \qwen & 0.88\textsubscript{\scriptsize{0.01}} & 0.69\textsubscript{\scriptsize{0.02}} & 0.77\textsubscript{\scriptsize{0.03}} \\
        \midrule
                \multicolumn{2}{c}{\textbf{Average}} & {0.67} & {0.68} & {0.70} \\
            \bottomrule
            \end{tabular}
        \end{adjustbox}
    \end{center}
\end{table}

%% file: tables/project_conflict.tex
\begin{table}[t]
\caption{\textbf{Normalized \texttt{AU} and \texttt{IF}:} The normalized \texttt{AU} (ratio of LLM's allocation to Modified UG) and \texttt{IF} scores for each PB instance with \textbf{project conflicts}. \textbf{Bold} numbers indicate the best-performing LLM.}
\label{tab:main_result_conflict}
    \begin{center}
        \begin{adjustbox}{max width=\linewidth}
            \begin{tabular}{ccccc}
            \toprule
          \multicolumn{2}{c}{\textbf{PB Setup}} & {\llamaThreeSeventyB} & {\llamanvidia}  & \qwen \\
            \midrule
     \multirow{2}{*}{\ppi}     &  \texttt{AU}  &  0.762\textsubscript{0.058}  & 0.792\textsubscript{0.000}  & \textbf{0.865\textsubscript{0.017}} \\
          &  \texttt{IF} &  0.681\textsubscript{0.020}  &  0.750\textsubscript{0.000}  & \textbf{0.778\textsubscript{0.020}} \\
    \multirow{2}{*}{\vrpi} &  \texttt{AU}  &   0.756\textsubscript{0.051}  &  0.746\textsubscript{0.000}  &  \textbf{0.865\textsubscript{0.020}} \\ 
     &  \texttt{IF}  &    0.528\textsubscript{0.071}  &  0.375\textsubscript{0.000} &  \textbf{0.681\textsubscript{0.039}} \\ 
            \bottomrule
            \end{tabular}
        \end{adjustbox}
    \end{center}
\end{table}

%% file: sections/appendix.tex
\section{\textsc{Pabulib}: \underline{PA}rticipatory \underline{BU}dgeting \underline{LIB}rary}
\setcounter{figure}{0}
\setcounter{table}{0}
\renewcommand{\thefigure}{\thesection.\arabic{figure}}
\renewcommand{\thetable}{\thesection.\arabic{table}}

\label{app::pabulib}

\paragraph{PB Instance.} A typical PB instance is a single UTF-8 encoded text file with a "\texttt{.pb}" extension, divided into three sections. The \texttt{META} section includes general metadata, such as the country, budget, and the number of votes. The \texttt{PROJECTS} section details project costs, categories, and targets. The \texttt{VOTES} section records the votes, which can be of four types -- approval, ordinal, cumulative, or scoring -- and may also include voter metadata, such as age and sex. 

Figure~\ref{fig::pb_instance} depicts a simple \textsc{Pabulib} instance. Each section begins with a title (\texttt{META}, \texttt{PROJECTS}, or  \texttt{VOTES} ), followed by a semicolon-separated column header line (e.g., \texttt{key;value} for \texttt{META}), and subsequent lines containing the corresponding data.

\paragraph{Utilitarian-Greedy (UG) Algorithm}~\citep{rey2023computational}. The UG algorithm aims to maximize the total utility of selected projects within a budget constraint. It starts with an empty set \( \mathbb{W} \) and iteratively adds the project \( p \) that maximizes the utility-to-cost ratio \( \sum_{i \in \mathbb{V}} \frac{u_i(p)}{\texttt{cost}(p)} \), provided that the total cost remains within the budget \( \beta \). The process continues until no more projects can be added without exceeding the budget. This greedy approach is optimal for selecting up to one project for maximizing overall utility~\citep{dantzig1957discrete}. Algorithm~\ref{algo::ug} provides the formal description.

\begin{algorithm}
\renewcommand{\thealgorithm}{\thesection.\arabic{algorithm}}
\caption{Utilitarian Greedy (UG) Algorithm}\label{algo::ug}
\begin{algorithmic}[1]
\STATE Initialize $\mathbb{W} \gets \emptyset$ (an empty set of projects)
\WHILE{there exists a project $p$ such that $\texttt{cost}(W) + \texttt{cost}(p) \leq b$}
    \STATE Select project $p$ maximizing $\sum_{i \in \mathbb{V}} \frac{u_i(p)}{\texttt{cost}(p)}$
    \STATE Add project $p$ to $W$, i.e., $\mathbb{W} \gets \mathbb{W} \cup \{p\}$
\ENDWHILE
\STATE \textbf{Return} $\mathbb{W}$
\end{algorithmic}
\end{algorithm}

\begin{figure*}
    \renewcommand{\thefigure}{\thesection.\arabic{figure}}
    \centering
    \fbox{%
        \begin{minipage}{0.75\textwidth}
        \scriptsize
        \texttt{META\\
        key; value\\
        description; Municipal PB in Wieliczka\\
        country; Poland\\
        unit; Wieliczka\\
        instance; 2020\\
        num\_projects; 5\\
        num\_votes; 10\\
        budget; 2500\\
        rule; greedy\\
        vote\_type; approval\\
        min\_length; 1\\
        max\_length; 3\\
        date\_begin; 16.09.2023\\
        date\_end; 04.10.2023\\
        \\    
        PROJECTS\\
        project\_id; cost; category\\
        1; 600; culture, education\\
        2; 800; sport\\
        4; 1400; culture\\
        5; 1000; health, sport\\
        7; 1200; education\\
        \\
        VOTES\\
        voter\_id; age; sex; vote\\
        1; 34; f; 1,2,4\\
        2; 51; m; 1,2\\
        3; 23; m; 2,4,5\\
        4; 19; f; 5,7\\
        5; 62; f; 1,4,7\\
        6; 54; m; 1,7\\
        7; 49; m; 5\\
        8; 27; f; 4\\
        9; 39; f; 2,4,5\\
        10; 44; m; 4,5\\
        }
        \end{minipage}
    }
    \caption{A Simple \textsc{Pabulib} Instance Example~\citep{faliszewski2023participatory}}\label{fig::pb_instance}
\end{figure*}

%%%%%%%%%%%
%%%%%%%%%%%%\
\input{sample_outputs/turk_pb_instance}
%%%%%%%%%%%%%
%%%%%%%%%%%%

\paragraph{Sample \textsc{Pabulib} Mechanical Turk PB Instance.} Figure~\ref{fig::pb_instance_turk} presents an instance from the \textsc{Pabulib} Mechanical Turk dataset, which includes details on projects and votes in a participatory budgeting experiment. It showcases 10 projects with varying costs and descriptions, along with 75 voter responses, including demographics such as age, sex, and education level.

\section{Additional Details and Experiments}
\label{app::more_exps}
\setcounter{figure}{0}
\setcounter{table}{0}
\renewcommand{\thefigure}{\thesection.\arabic{figure}}
\renewcommand{\thetable}{\thesection.\arabic{table}}

\subsection{Hyperparameters}
\paragraph{Hugging Face Text Generation Pipeline}~\citep{huggingface_transformers_pipelines}. For all open-source models, we use the \emph{Hugging Face} text generation pipeline~\citep{huggingface_transformers_pipelines} with default hyperparameters for decoding. Specifically, the decoding strategy employs greedy decoding, where the model selects the highest-probability token at each step, minimizing randomness in the generated output. Additionally, the default setup implies that only a single response is generated for each input. This approach minimizes randomness in LLM's output.

\paragraph{Open AI Text Generation API}~\citep{OpenAI_Text_Generation_Guide}. For OpenAI models, we utilized the OpenAI Text Generation API via the \texttt{ChatCompletion} endpoint, specifying the model name, message history, and a maximum completion length of $150$ tokens. As noted in the main paper, we did not prompt OpenAI models to explain their allocation to reduce costs. We set \texttt{top\_p} to 1, disabling nucleus sampling, and \texttt{n} to 1, ensuring a single response. We used OpenAI’s default greedy decoding strategy to minimize randomness, which selects the highest-probability tokens at each step.

Future work can explore the impact of different decoding strategies on an LLM's performance and instruction-following ability.

\subsection{System Prompts}
Our system prompts for few-shot prompting of LLMs vary depending on the type of PB instance and setting. For the standard PB instance, \texttt{PPI}, and its natural language variant, \texttt{NLVPI}, the system prompt outlines the setup, the utility score, and the UG algorithm. For \texttt{VRPI}, where LLMs must predict votes, the prompt directs the model to focus on voter metadata to infer preferences and requires the output to be provided as a \texttt{json} object. In the case of PB with project conflicts, the system shows an example of conflicting projects and asks the LLM to produce a utility-maximizing allocation, without explicitly guiding it to follow a greedy approach. Detailed system prompts for each scenario are available in Appendix~\ref{app::system_prompts}.

\subsection{Additional Experiments\label{app::subsec_add_exps}}

\subsubsection{Ablations}
We ablate our findings on \emph{model size}.

\paragraph{Model Size: 7B-14B.}
 We conduct experiments on comparatively smaller instruction fine-tuned LLMs, including \llamaEight, \qwenOneFour, and \mistralSeven.

%%%%
%%%%%
\input{tables/ablation}

%%%%%
%%%%%

As shown in Tables~\ref{tab:main_result_if_app} and~\ref{tab:main_result_if_app}, these models generally exhibit lower \texttt{IF} and \texttt{AU} scores compared to their larger counterparts. Notably, \qwenOneFour\ demonstrates relatively stronger performance in terms of both instruction-following and utility, achieving the highest \texttt{IF} in \texttt{VRPI} and \texttt{PPI} settings, as well as the best \texttt{AU} in \texttt{PPI} and \texttt{VRPI}. Interestingly, these relatively smaller models struggle with inferring voter preferences, as indicated by their lower \texttt{AU} in \texttt{VRPI} and \texttt{NLVPI}, they achieve competitive performance in \texttt{PPI}, suggesting that they find it easier to process explicit voter preferences rather than infer them from natural language descriptions.

\paragraph{Model Size: 3B.}
In this set of experiments, we focus on the performance of smaller 3B models, specifically \llamaThree\ and \qwenThree. As shown in Table~\ref{tab:main_result_if_app_2} and Table~\ref{tab:main_result_au_app_2}, \qwenThree\ outperforms \llamaThree\ in terms of consistency, achieving a higher fraction of consistent allocations in all PB setups. The average fraction across setups for \qwenThree\ is 0.639 for \texttt{IF}, compared to 0.236 for \llamaThree. Similarly, \qwenThree\ exhibits higher normalized utility scores across all PB types, particularly in the \texttt{PPI} setup, where it reaches 0.655, compared to 0.249 for \llamaThree. The average normalized utility across setups for \qwenThree\ is 0.605, significantly higher than \llamaThree\ at 0.268. These results highlight the comparative strength of \qwenThree\ over \llamaThree\ in both consistency and utility. 

%%%%%
\input{tables/ablation_2}

%%%%%

\subsubsection{Fraction Overlap}
We also look at the PB instance-wise mean Fraction overlap with \texttt{PPI} as ground truth for the top 25\% and 75\% highest-voted projects.

\begin{table}[ht]
\centering
\caption{Normalized top-25\% Fraction Overlap, per method, for each LLM. We also report the Normalized Utility (AU) from Table~\ref{tab:main_result_au}.}
\label{tab:overlap_au_2}
\begin{adjustbox}{max width=\columnwidth}
\begin{tabular}{lccc}
\toprule
\textbf{Model} &   \textbf{Method}  & \textbf{Overlap} & \textbf{AU} \\
\midrule
\multirow{3}{*}{\llamaThreeSeventyB} &  \greedy & 0.342 & 0.573 \\
& \optimization  & 0.319  & 0.773 \\
& \hillclimb & 0.361  &  0.694 \\  
\midrule
\multirow{3}{*}{\llamanvidia }   & \greedy & 0.344 & 0.620 \\
& \optimization & 0.324 &  0.656 \\
& \hillclimb & 0.305 &  0.704 \\
\midrule
\multirow{3}{*}{\qwen }              & \greedy & 0.365 & 0.574 \\
& \optimization & 0.453 & 0.820 \\
& \hillclimb & 0.472 & 0.694 \\
\bottomrule
\end{tabular}
\end{adjustbox}
\end{table}

\begin{table}[ht]
\centering
\caption{Normalized top-50\% Fraction Overlap, per method, for each LLM. We also report the Normalized Utility (AU) from Table~\ref{tab:main_result_au}.}
\label{tab:overlap_au_3}
\begin{adjustbox}{max width=\columnwidth}
\begin{tabular}{lccc}
\toprule
\textbf{Model} &   \textbf{Method}  & \textbf{Overlap} & \textbf{AU} \\
\midrule
\multirow{3}{*}{\llamaThreeSeventyB} &  \greedy & 0.582 & 0.573 \\
& \optimization  & 0.547 & 0.773 \\
& \hillclimb &  0.602&  0.694 \\  
\midrule
\multirow{3}{*}{\llamanvidia }   & \greedy & 0.583 & 0.620 \\
& \optimization & 0.576 &  0.656 \\
& \hillclimb & 0.581 &  0.704 \\
\midrule
\multirow{3}{*}{\qwen }              & \greedy & 0.445 & 0.574 \\
& \optimization & 0.617 & 0.820 \\
& \hillclimb & 0.594 & 0.694 \\
\bottomrule
\end{tabular}
\end{adjustbox}
\end{table}

\paragraph{Top-25\%: }As shown in Table~\ref{tab:overlap_au_2}, for the top 25\% overlap, \llamaThreeSeventyB\ achieves 0.361 with a corresponding \texttt{AU} of 0.694. \llamanvidia\ reaches an overlap of 0.305 and \texttt{AU} of 0.704. \qwen\ outperforms both, with the highest overlap of 0.472 and a matching \texttt{AU} of 0.694.

\paragraph{Top-50\%:} In Table~\ref{tab:overlap_au_3}, for the top 50\% overlap, \llamaThreeSeventyB\ records 0.602 with an \texttt{AU} of 0.694. \llamanvidia\ shows a comparable overlap of 0.581 and the same \texttt{AU}. Again, \qwen\ achieves the highest overlap of 0.594, maintaining an \texttt{AU} of 0.694.

\subsubsection{PB with Project Conflicts}

%%%%%%%%%%%%%%%%%%%%%%%%%%%%%%%
%%%%%%%%%%%%%%%%%%%%%%%%%%%%%%%
%%%%%%%%%%%%%%%%%%%%%%%%%%%%%%%
\begin{figure*}[t]
    \centering
    \renewcommand{\thefigure}{\thesection.\arabic{figure}}
    \includegraphics[width=\linewidth]{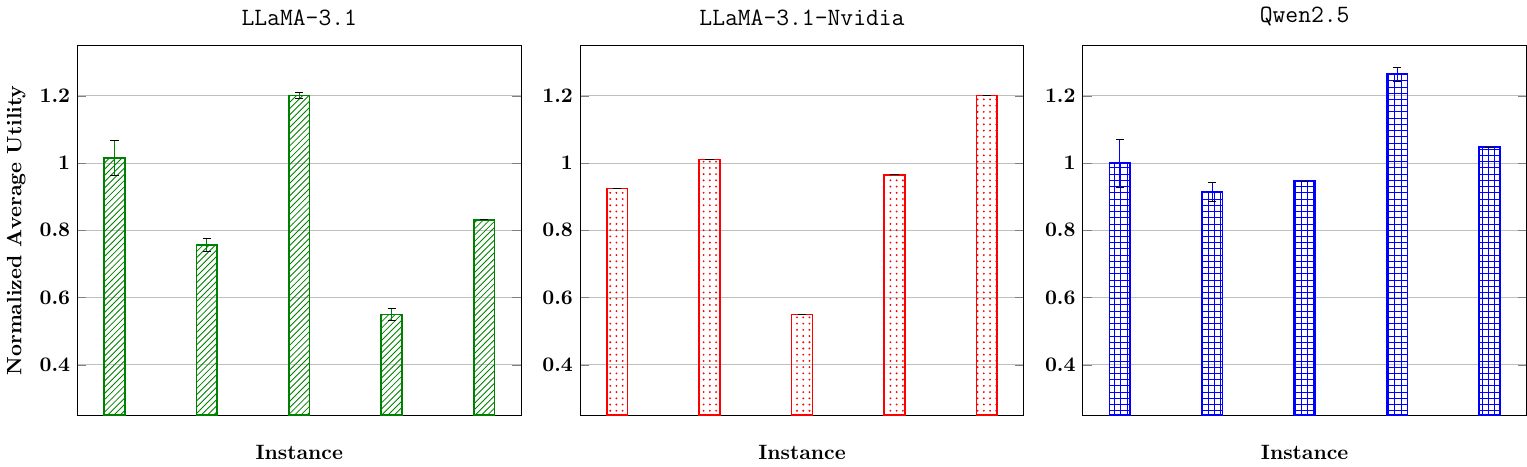}
    \caption{Normalized Average Utility (AU) for five sampled Participatory Budgeting (PB) instances with project conflicts. The plot compares the performance of different models (\texttt{LLaMA-3.1}, \texttt{LLaMA-3.1-Nvidia}, and \texttt{Qwen2.5}) across multiple instances. Instances showing AU values greater than 1 indicate better performance than the Modified UG algorithm.}
    \label{fig:project_conflict_instances}
\end{figure*}
%%%%%%%%%%%%%%%%%%%%%%%%%%%%%%%
%%%%%%%%%%%%%%%%%%%%%%%%%%%%%%%
%%%%%%%%%%%%%%%%%%%%%%%%%%%%%%%
%%%%%%%%%%%%%%%%%%%%%%%%%%%%%%%
Figure \ref{fig:project_conflict_instances} presents the normalized average utility (AU) for five sampled participatory budgeting (PB) instances with project conflicts across three different models: \texttt{LLaMA-3.1}, \texttt{LLaMA-3.1-Nvidia}, and \texttt{Qwen2.5}. Each model's performance is visualized with error bars, showing the variance in AU across instances. \texttt{LLaMA-3.1-Nvidia} consistently achieves the highest utility, followed by \texttt{Qwen2.5} and \texttt{LLaMA-3.1}, suggesting variations in model efficiency when handling project conflicts in PB contexts. Notably, in some instances, the AU exceeds 1, indicating that these models outperform the modified UG algorithm in terms of utility.

\subsubsection{Sample Outputs\label{app::sample_outputs}}

%%%%%%%%%%%%%%%%%%%%%%%%%%%%%%%
%%%%%%%%%%%%%%%%%%%%%%%%%%%%%%%
\input{sample_outputs/output_greedy_basic}
%%%%%%%%%%%%%%%%%%%%%%%%%%%%%%%
%%%%%%%%%%%%%%%%%%%%%%%%%%%%%%%

%%%%%%%%%%%%%%%%%%%%%%%%%%%%%%%
%%%%%%%%%%%%%%%%%%%%%%%%%%%%%%%
\input{sample_outputs/output_greedy_good}
%%%%%%%%%%%%%%%%%%%%%%%%%%%%%%%
%%%%%%%%%%%%%%%%%%%%%%%%%%%%%%%

Figures \ref{fig:output_greedy_base} and Figure~\ref{fig:output_greedy_good} compare the UG allocations produced by \llamaThreeSeventyB\ and \qwen\ on PB instances of type \texttt{PPI}. \llamaThreeSeventyB\ stops prematurely, allocating only USD 137,000 compared to the baseline UG allocation of USD 481,400, while \qwen\ fully utilizes the budget and matches the baseline allocation. Despite this difference, both models loosely follow the greedy order rather than strictly adhering to it.

\section{Proofs \label{app::proofs}}
\setcounter{figure}{0}
\setcounter{table}{0}
\renewcommand{\thefigure}{\thesection.\arabic{figure}}
\renewcommand{\thetable}{\thesection.\arabic{table}}

\subsection{Proposition~\ref{prop::welfare_ratio}}

\begin{proposition}\label{prop::welfare_ratio}
    There exist instances of the PB game $\langle \mathbb{P}, \mathbb{V}, \beta, \texttt{cost} \rangle$ such that the ratio of the optimal social welfare to the welfare of the Modified UG algorithm is unbounded. Specifically, for some family of instances parameterized by $n=|\mathbb{V}|$, we have  
    $$
        \lim_{n \to \infty} \frac{\sum_{i\in\mathbb{V}} u_i(\mathbb{W_\texttt{opt}})}{\sum_{i\in\mathbb{V}}u_i(\mathbb{W}_{\texttt{greedy}}\})} = \infty.
    $$
\end{proposition}

\begin{proof}

    To prove
    \[
    \frac{\sum_{i\in\mathbb{V}}u_i(\mathbb{W_\texttt{opt}})}{\sum_{i\in\mathbb{V}}u_i(\mathbb{W_\texttt{greedy}})} \to \infty \quad \text{as } n \to \infty.
    \]
    we use proof by construction. Specifically, let $\mathbb{P}=\{p_1,p_2,p_3\}$ be the set of three available projects, $\mathbb{V}=[n]$ the set of voters, and $\beta$ as the total budget. The cost of the projects are: $\texttt{cost}(p_1)=\texttt{cost}(p_2)=\epsilon$ for an arbitrary $\epsilon > 0$ and $\texttt{cost}(p_3)=\beta$. Crucially, the conflict set $\mathbb{C}=\mathbb{P}$ implying that only one of the three projects can be provisioned. Finally, the votes are: voter $1$ votes only for $p_1$, voter $2$ votes only for $p_2$ and the remaining $(n-2)$ voters vote only for $p_3$.

    For this setting, it is easy to see that provisioning $p_3$ is optimal, i.e., $\mathbb{W_\texttt{opt}}=\{p_3\}$. This generates the optimal social welfare $\sum_{i\in \mathbb{V}} u_i(\mathbb{W_\texttt{opt}}) = n-2$. The modified UG algorithm will order projects as follows: $p_1$ or $p_2$ will appear first in the greedy order as their $\frac{\sum_{i\in \mathbb{V}} u_i(p_1)}{\texttt{cost}(p_1)} =\frac{\sum_{i\in \mathbb{V}} u_i(p_2)}{\texttt{cost}(p_2)}= \frac{1}{\epsilon}$ can be made arbitrary more than $p_3$'s greedy score of $\frac{\sum_{i\in \mathbb{V}} u_i(p_3)}{\texttt{cost}(p_3)} = \frac{n-2}{\beta}$. 
    
    W.l.o.g, we can assume that the tie is broken in $p_1$'s favor. The Modified UG algorithm selects $p_1$ and, subsequently, cannot select $p_3$; or, $\mathbb{W_\texttt{greedy}}=\{p_1\}$. This leads to $\sum_{i\in\mathbb{V}}u_i(\mathbb{W_\texttt{greedy}})=1$. That is,
    \[
    \frac{\sum_{i\in\mathbb{V}}u_i(\mathbb{W_\texttt{opt}})}{\sum_{i\in\mathbb{V}}u_i(\mathbb{W_\texttt{greedy}})} = \frac{n-2}{1} \to \infty \quad \text{as } n \to \infty.
    \]
This proves the theorem.
\end{proof}

\subsection{Proposition~\ref{prop::upto_one}}

% \smallskip

\begin{proposition}\label{prop::upto_one}
    The Modified UG algorithm is not {optimal up to one project}. That is, there exist instances of the PB game $\langle \mathbb{P}, \mathbb{V}, \beta, \texttt{cost} \rangle$ and any project $p \in \mathbb{P} \setminus \mathbb{W}_{\texttt{greedy}}$ such that
    $$
        \sum_{i\in\mathbb{V}} u_i\left(\mathbb{W}_{\texttt{greedy}} \cup \{p\}\right) < \max_{\mathbb{W}^\prime } \sum_{i\in\mathbb{V}} u_i\left(\mathbb{W}^\prime\right),
    $$
    where $\mathbb{W}^\prime$ is over all feasible allocations~(Eq.~\ref{eqn::conflict}).
\end{proposition}

\begin{proof} We again use proof by construction. Consider a PB game with the following setup:
    \begin{itemize}[leftmargin=*,noitemsep]
        \item A budget $\beta = 2$.
        \item A set of projects $\mathbb{P} = \{p_1, p_2, p_3\}$ with costs:
        \[
            \texttt{cost}(p_1) = 1, \quad \texttt{cost}(p_2) = 1, \quad \texttt{cost}(p_3) = 2.
        \]
        \item A set of $n$ voters $\mathbb{V} = \{1, \dots, n\}$ with preferences:
        \begin{itemize}[leftmargin=*]
            \item A set of $\frac{n}{2}$ voters approve $\{p_1\}$.
            \item Another set of $\frac{n}{5}$ voters approve $\{p_2\}$.
            \item A set of $\frac{4n}{5}$ voters approve $\{p_3\}$. 
        \end{itemize}
        Each voter gets a utility of one if one of their approved projects is provisioned and zero otherwise.
        \item A conflict set $\mathbb{C} = \{p_1, p_2\}$, meaning $p_1$ and $p_2$ cannot be selected together.
    \end{itemize}
    
    The Modified UG algorithm selects projects greedily based on approval count and cost while respecting conflicts. That is, the greedy order is:
    \begin{enumerate}[leftmargin=*]
        \item $p_1$ as $\frac{\sum_{i\in \mathbb{V}} u_i(p_1)}{\texttt{cost}(p_1)} = \frac{n}{2}$
        \item $p_3$ as $\frac{\sum_{i\in \mathbb{V}} u_i(p_3)}{\texttt{cost}(p_3)} = \frac{2n}{5}$
        \item $p_2$ as $\frac{\sum_{i\in \mathbb{V}} u_i(p_2)}{\texttt{cost}(p_2)} = \frac{n}{5}$
    \end{enumerate}
    
    The greedy algorithm first picks $p_1$. Next, it cannot pick $p_3$ after $p_1$ as that overshoots the budget. It cannot also pick $p_2$ after picking $p_1$ as both are part of $\mathbb{C}$. Thus, $\mathbb{W}_{\texttt{greedy}}=\{p_1\}$.
    The total utility of this allocation is:
    \[
        \sum_{i \in \mathbb{V}} u_i(\mathbb{W_\texttt{greedy}}) = \frac{n}{2}.
    \]

    Trivially, the optimal allocation is $\mathbb{W}_\texttt{opt} = \{p_3\}$, achieving a total utility of:
    \[
        \sum_{i \in \mathbb{V}} u_i(\mathbb{W}_\texttt{opt}) = \frac{4n}{5} > u_i(\mathbb{W_\texttt{greedy}}).
    \]

    The only project we can add to $\mathbb{W_\texttt{greedy}}$ is $p_2$. This results in the following,
    \[
        \sum_{i \in \mathbb{V}} u_i(\mathbb{W_\texttt{greedy}} \cup \{p_2\}) = \frac{n}{2} + \frac{n}{5} = \frac{7n}{10} < u_i(\mathbb{W}_\texttt{opt}).
    \]

    Since $\mathbb{W_\texttt{greedy}} \cup \{p_2\}$ still does not achieve a strictly higher utility than $\mathbb{W}_\texttt{opt}$, the Modified UG algorithm is not optimal up to one project.

\end{proof}

\section{Prompts}
\setcounter{figure}{0}
\setcounter{table}{0}
\renewcommand{\thefigure}{\thesection.\arabic{figure}}
\renewcommand{\thetable}{\thesection.\arabic{table}}

\subsection{PB Instances: Example Prompts}

Figure~\ref{fig::pb_instance_vrpi} and Figure~\ref{fig::pb_instance_nlvpi} provide example prompts to explain the {Vote-Removed PB Instance} \texttt{(VRPI)} and {Natural Language Votes PB Instance} \texttt{(NLVPI)} introduced in Section~\ref{sec::method}. The \texttt{PPI} PB instance is the standard PB instance, already presented with Figure~\ref{fig::pb_instance}.

\begin{figure*}
    \centering
    \fbox{%
        \begin{minipage}{0.75\textwidth}
        \scriptsize
        \texttt{META\\
        key; value\\
        description; Municipal PB in Wieliczka\\
        country; Poland\\
        unit; Wieliczka\\
        instance; 2020\\
        num\_projects; 5\\
        budget; 2500\\
        rule; greedy\\
        vote\_type; approval\\
        min\_length; 1\\
        max\_length; 3\\
        date\_begin; 16.09.2023\\
        date\_end; 04.10.2023\\
        \\    
        PROJECTS\\
        project\_id; cost; category\\
        1; 600; culture, education\\
        2; 800; sport\\
        4; 1400; culture\\
        5; 1000; health, sport\\
        7; 1200; education\\
        \\
        VOTES\\
        voter\_id; age; sex\\
        1; 34; f\\
        2; 51; m\\
        3; 23; m\\
        4; 19; f\\
        5; 62; f\\
        6; 54; m\\
        7; 49; m\\
        8; 27; f\\
        9; 39; f\\
        10; 44; m\\
        }
        \end{minipage}
    }
    \caption{Example of {Vote-Removed PB Instance} \texttt{(VRPI)}}\label{fig::pb_instance_vrpi}
\end{figure*}

\begin{figure*}
    \centering
    \fbox{%
        \begin{minipage}{0.75\textwidth}
        \scriptsize
        \texttt{META\\
        key; value\\
        description; Municipal PB in Wieliczka\\
        country; Poland\\
        unit; Wieliczka\\
        instance; 2020\\
        num\_projects; 5\\
        num\_votes; 10\\
        budget; 2500\\
        rule; greedy\\
        vote\_type; approval\\
        min\_length; 1\\
        max\_length; 3\\
        date\_begin; 16.09.2023\\
        date\_end; 04.10.2023\\
        \\    
        PROJECTS\\
        project\_id; cost; category\\
        1; 600; culture, education\\
        2; 800; sport\\
        4; 1400; culture\\
        5; 1000; health, sport\\
        7; 1200; education\\
        \\    
        VOTES\\
        voter\_id; age; sex; votes\\
        1; 34; f; I voted for Projects 1, 2, and 4, supporting culture, education, and sport initiatives.\\
        2; 51; m; I voted for Projects 1 and 2, focusing on culture and sports activities.\\
        3; 23; m; I voted for Projects 2, 4, and 5, emphasizing his interest in sports, culture, and health projects.\\
        4; 19; f; I voted for Projects 5 and 7, prioritizing health and education.\\
        5; 62; f; I voted for Projects 1, 4, and 7, highlighting her support for culture, education, and health.\\
        6; 54; m; I voted for Projects 1 and 7, showing a preference for culture and education.\\
        7; 49; m; I voted for Project 5, emphasizing health and sports.\\
        8; 27; f; I voted for Project 4, showing interest in cultural projects.\\
        9; 39; f; I voted for Projects 2, 4, and 5, indicating her preference for sports, culture, and health initiatives.\\
        10; 44; m; I voted for Projects 4 and 5, supporting culture and health programs.\\
        }
        \end{minipage}
    }
    \caption{{Natural Language Votes PB Instance} \texttt{(NLVPI)}}\label{fig::pb_instance_nlvpi}
\end{figure*}

\subsection{System Prompts\label{app::system_prompts}}
Figure~\ref{fig::system_prompt_ppi} and Figure~\ref{fig::pb_instance_vrpi} provide the system prompts for our first set of experiments. These are the experiments on different types of PB instances, i.e., \texttt{PPI}, \texttt{VRPI}, and \texttt{NLVPI}. Moreover, Figure~\ref{fig::system_prompt_conflict} provides the system prompt for PB with Project Conflicts, where we show the LLM an example for conflicting projects, and instruct it to provide a utility-maximizing allocation without explicitly prompting it to follow any greedy approach.

%%%%%%%%%%%%%%%%%%
%%%%%%%%%%%%%%%%%%
\input{system_prompts/ppi}
%%%%%%%%%%%%%%%%%%
%%%%%%%%%%%%%%%%%%

%%%%%%%%%%%%%%%%%%
%%%%%%%%%%%%%%%%%%
\input{system_prompts/vrpi}
%%%%%%%%%%%%%%%%%%
%%%%%%%%%%%%%%%%%%

%%%%%%%%%%%%%%%%%%
%%%%%%%%%%%%%%%%%%
\input{system_prompts/conflict}
%%%%%%%%%%%%%%%%%%
%%%%%%%%%%%%%%%%%%

\section{Resources Used}
\label{sec:resources_used}
We used 4 NVIDIA A100 GPUs (up to 700 GB RAM) with a 96-core Intel CPU to run our inferences. The list of LLMs, their URL with checkpoints available, and their licenses are listed below:
\begin{itemize}[leftmargin=*]
    \item \llamaThreeSeventyB. \textbf{URL:} \url{huggingface.co/meta-llama/Llama-3.1-70B-Instruct} {\bf License:} llama3.1
    
    \item \llamaEight. \textbf{URL:} \url{huggingface.co/meta-llama/Llama-3.1-8B-Instruct} {\bf License:} llama3.1
    
    \item \llamaThree. \textbf{URL:} \url{huggingface.co/meta-llama/Llama-3.2-3B-Instruct} {\bf License:} llama3.2
    
    \item \mistralSeven. \textbf{URL:} \url{huggingface.co/mistralai/Mistral-7B-v0.1} {\bf License:} apache-2.0
    
    \item \qwen. \textbf{URL:} \url{huggingface.co/Qwen/Qwen2.5-72B-Instruct} {\bf License:} qwen

    \item \qwenOneFour. \textbf{URL:} \url{huggingface.co/Qwen/Qwen2.5-14B-Instruct} {\bf License:} apache-2.0
    
    \item \qwenThree. \textbf{URL:} \url{huggingface.co/Qwen/Qwen2.5-3B-Instruct} {\bf License:} qwen-research

    \item \gptMini, \gptFour\ \& \gptThreeFive. \textbf{URL:} \url{platform.openai.com/docs/models} {\bf License:}  OpenAI API
\end{itemize}

%% file: sample_outputs/turk_pb_instance.tex
\begin{figure*}
    \renewcommand{\thefigure}{\thesection.\arabic{figure}}
    \centering
    \fbox{%
        \begin{minipage}{0.75\textwidth}
        \scriptsize
        \texttt{META\\
key;value\\
description;Experimental PB on Mechanical Turk | Ranking\_value\_money\_3\\
country;Worldwide\\
unit;Mechanical Turk\\
instance;Ranking\_value\_money\_3\\
num\_projects;10\\
num\_votes;75\\
budget;500000\\
vote\_type;ordinal\\
date\_begin;2022\\
date\_end;2022\\
min\_length;10\\
max\_length;10\\
edition;1\\
language;en\\
currency;USD\\
experimental;1\\
acknowledgments;The dataset was created in an experiment as part of the paper "Participatory Budgeting Design for the Real World" by Roy Fairsetin, Gerdus Benade and Kobi Gal.\\
PROJECTS\\
project\_id;cost;votes;score;name;category;description\\
3;27000;75;558;Computers for the community learning center;Culture \& community;Funding 20 laptops including mice and keyboards, giving students a place to study. \\
25;90000;75;524;The Sustainable Energy Pilot;Environment, public health \& safety;Install energy conversion devices on gym equipment and a rapid electric vehicle charging station.\\
13;24000;75;498;Real-Time Bus Arrival Monitors in bus stations;Streets, Sidewalks \& Transit;Real-time bus arrival monitors bus stops will inform travelers when the next bus will arrive, so they can adjust their plans if needed.\\
51;105000;75;462;Security Cameras;Education;Install security cameras in public schools.\\
22;320000;75;426;24H public toilet;Environment, public health \& safety;24-hour access public toilet near Central Square.\\
7;50000;75;382;Laundry Access in Public Schools;Culture \& community;Renovate a space in a Cambridge Public School and install washers and dryers for students who do not have easy access to laundry services at home, to use for their clothing and necessities.\\
40;120000;75;363;Let’s Rest: Picnic Tables \& Benches for Our Parks;Facilities, parks \& recreation;Benches and picnic tables bring our community together. Installing new benches and picnic tables in up to 10 of our park will allow people of all ages and abilities to enjoy them for resting, talking, reading, people watching and being outdoors (3 of which are shown in the map).\\
16;90000;75;340;Sheltered Bike Parking at the Main Library;Streets, Sidewalks \& Transit;The Main Library needs more bicycle parking. A glass pavilion, protecting bikes from the weather, landscaped with paths and trees, will be an attractive and functional addition to the library grounds.\\
45;250000;75;308;Installing Lights at the school Basketball Court;Education;Install lighting to extend safe playing hours for basketball courts. Increases safety for community members while expanding healthy alternatives for youth and access to public space.\\
34;250000;75;264;Dog Park;Facilities, parks \& recreation;Building a dog park.\\
VOTES\\
voter\_id;vote;age;sex;education\\
1025;51,3,13,22,7,25,40,34,45,16;30;F;Graduate degree\\
1026;51,25,3,13,16,22,45,40,7,34;44;M;Graduate degree\\
1027;3,13,22,40,25,7,16,34,45,51;30;M;Graduate degree\\
1029;51,40,7,45,13,3,22,16,25,34;65;F;Graduate degree\\
1030;51,7,25,45,40,22,3,16,34,13;39;M;College\\
2313;22,3,51,25,16,7,40,13,45,34;65;M;Graduate degree\\
2441;45,13,25,40,22,51,3,7,16,34;35;M;Graduate degree\\
1035;22,3,13,40,16,34,45,51,25,7;32;M;Graduate degree\\
2495;3,13,25,22,40,16,7,51,45,34;50;M;High School/GED\\
1037;13,34,45,40,16,22,25,3,51,7;46;F;Graduate degree\\
2189;22,45,34,51,25,16,7,3,13,40;37;M;College\\
1680;3,40,25,13,7,45,34,22,16,51;58;M;College\\
1042;3,22,25,40,13,7,51,34,45,16;31;M;Graduate degree\\
1043;22,3,16,25,13,51,45,7,40,34;31;M;Graduate degree\\
2066;7,13,40,22,3,16,25,45,51,34;54;M;Graduate degree\\
1817;3,13,40,25,34,22,51,16,45,7;70;F;Graduate degree\\
2462;13,51,25,34,3,22,45,40,7,16;38;F;College\\
929;25,3,7,13,51,16,45,34,40,22;39;M;Graduate degree\\
2337;25,22,3,7,34,40,51,13,16,45;27;M;College\\
2211;13,3,7,25,16,45,40,34,51,22;32;M;College\\
1788;3,7,13,25,16,22,45,40,34,51;24;M;College\\
1578;51,22,25,13,3,7,45,16,40,34;38;F;Graduate degree\\
2093;51,45,3,22,40,25,13,7,16,34;38;F;College\\
1597;34,13,22,45,40,51,16,25,7,3;57;F;College\\
2365;22,3,7,45,34,40,16,51,25,13;33;F;High School/GED\\
959;3,25,51,13,16,7,40,34,22,45;32;M;Graduate degree\\
960;40,3,22,34,51,25,16,13,7,45;23;M;High School/GED\\
1472;25,22,34,16,13,45,51,7,3,40;28;M;Graduate degree\\
1473;25,13,3,40,16,7,22,45,51,34;50;M;Graduate degree\\
1474;13,3,40,16,34,22,45,51,7,25;45;M;Graduate degree\\
1475;13,34,22,25,51,3,45,7,16,40;18;M;High School/GED\\
1477;3,13,7,25,16,45,22,51,40,34;26;M;High School/GED\\
1478;25,13,3,51,7,22,40,34,16,45;24;M;College\\
1479;25,34,51,3,40,13,7,45,16,22;24;M;High School/GED\\
1480;16,13,25,51,22,3,45,7,40,34;21;M;High School/GED\\
\textcolor{blue}{(other voters omitted)}
}
        \end{minipage}
    }
    \caption{An Instance among the 24 \textsc{Pabulib} Mechanical Turk Instances~\citep{faliszewski2023participatory}}\label{fig::pb_instance_turk}
\end{figure*}

%% file: tables/ablation.tex
\begin{table*}[t]
\caption{\texttt{IF} Comparison: The fraction of PB instances for which the LLM produces a consistent allocation. The numbers in \textbf{bold} highlight the highest fraction, for each LLM.}
\label{tab:main_result_if_app}
    \begin{center}
        \begin{adjustbox}{max width=0.9\linewidth}
            \begin{tabular}{ccccc}
            \toprule 
            \multicolumn{1}{c}{\bf Setup}  & \llamaEight & \qwenOneFour & \mistralSeven & \textbf{Average} (across Models) \\
            \midrule
            \texttt{PPI}  & \textbf{0.694 $\pm$ 0.079} &\textbf{ 0.750 $\pm$ 0.034} & 0.833 $\pm$ 0.00 &  0.759 \\
            \texttt{VRPI} &  0.569 $\pm$ 0.071 & \textbf{0.750 $\pm$ 0.034} & 0.625 $\pm$ 0.00 & 0.648 \\
            \texttt{NLVPI} &  0.596 $\pm$ 0.099  & 0.561 $\pm$ 0.066 & 0.526 $\pm$ 0.00 & 0.561\\
            \textbf{Average} (across Setup) &  0.620  &  0.687   &  0.661 & -- \\
            \bottomrule
           \end{tabular}
        \end{adjustbox}
    \end{center}
\end{table*}

\begin{table*}[t]
\caption{Normalized \texttt{AU} Comparison: The average utility for each PB instance (normalized by the UG allocation) produced by the LLM. The numbers in \textbf{bold} highlight the highest normalized utility, for each LLM.}
\label{tab:main_result_au_app}
    \begin{center}
        \begin{adjustbox}{max width=0.9\linewidth}
            \begin{tabular}{ccccc}
            \toprule 
            \multicolumn{1}{c}{\bf Setup}  & \llamaEight & \qwenOneFour & \mistralSeven & \textbf{Average}  (across Models) \\
            \midrule
            \texttt{PPI}  & 0.520 $\pm$ 0.011 & \textbf{0.633 $\pm$  0.042} & \textbf{0.439 $\pm$ 0.00} & 0.531 \\
            \texttt{VRPI} & 0.423 $\pm$ 0.026 & 0.579 $\pm$ 0.029 & 0.388 $\pm$ 0.00 & 0.463  \\
            \texttt{NLVPI} &  \textbf{0.531 $\pm$ 0.055} & 0.587 $\pm$ 0.030 & 0.394 $\pm$ 0.00 & 0.504\\
            \textbf{Average} (across Setup) &   0.491 &  0.600   &   0.407 & -- \\
            \bottomrule
           \end{tabular}
        \end{adjustbox}
    \end{center}
\end{table*}

%% file: tables/ablation_2.tex
\begin{table*}[t]
\caption{\texttt{IF} Comparison: The fraction of PB instances for which the LLM (i) follows the intended output format and (ii) produces an allocation within the budget.  The numbers in \textbf{bold} highlight the highest fraction, for each LLM.}
\label{tab:main_result_if_app_2}
    \begin{center}
        \begin{adjustbox}{max width=0.8\linewidth}
            \begin{tabular}{ccccc}
            \toprule 
            \multicolumn{1}{c}{\bf Setup}  & \llamaThree & \qwenThree & \textbf{Average} (across Models) \\
            \midrule
            \texttt{PPI}  & 0.361 $\pm$ 0.030 &  0.708 $\pm$ 0.102 &  0.535  \\
            \texttt{VRPI} & 0.125 $\pm$ 0.068 &  0.639 $\pm$ 0.020 &  0.382   \\
            \texttt{NLVPI} &  0.222 $\pm$ 0.020 & 0.569 $\pm$ 0.039 & 0.396   \\
            \textbf{Average} (across Setup) & 0.236 & 0.639 &  --   \\
            \bottomrule
           \end{tabular}
        \end{adjustbox}
    \end{center}
\end{table*}

\begin{table*}[t]
\caption{Normalized \texttt{AU} Comparison: The average utility for each PB instance (normalized by the UG allocation) produced by the LLM. The numbers in \textbf{bold} highlight the highest normalized utility, for each LLM.}
\label{tab:main_result_au_app_2}
    \begin{center}
        \begin{adjustbox}{max width=0.8\linewidth}
            \begin{tabular}{cccc}
            \toprule 
            \multicolumn{1}{c}{\bf Setup}  & \llamaThree & \qwenThree & \textbf{Average}  (across Models) \\
            \midrule
            \texttt{PPI}  & 0.249 $\pm$ 0.046 & 0.655 $\pm$ 0.016 &  0.452  \\
            \texttt{VRPI} & 0.257 $\pm$ 0.108 & 0.625 $\pm$ 0.079 & 0.441   \\
            \texttt{NLVPI} & 0.298 $\pm$ 0.131 & 0.535 $\pm$ 0.069 &   0.417 \\
            \textbf{Average} (across Setup) & 0.268 & 0.605 &  --  \\
            \bottomrule
           \end{tabular}
        \end{adjustbox}
    \end{center}
\end{table*}

%% file: sample_outputs/output_greedy_basic.tex
%%%%%%%%%%%%%%%%%%%%%%%%%%%%%%%
%%%%%%%%%%%%%%%%%%%%%%%%%%%%%%%
%%%%%%%%%%%%%%%%%%%%%%%%%%%%%%%
\begin{figure*}[h]
    \centering
        \fbox{%
        \begin{minipage}{0.9\textwidth}
        \scriptsize
        \ttfamily
        \texttt{\{'allocation': [31, 2, 23, 13, 14, 3], 'explanation': "  Projects: - ID: 3, Cost: 27,000, Votes: 48 - ID: 14, Cost:\
      \ 25,000, Votes: 41 - ID: 13, Cost: 24,000, Votes: 40 - ID: 31, Cost: 8,000,\
      \ Votes: 38 - ID: 2, Cost: 13,000, Votes: 39 - ID: 23, Cost: 40,000, Votes:\
      \ 36  Steps: 1. Calculate the utility-to-cost ratio for each project. 2.\
      \ Sort projects by this ratio in descending order. 3. Select projects one by\
      \ one, checking the budget constraint.   Output: 1. Instance: Knapsack\_7 \
      2. Selected Projects: [31, 2, 23, 13, 14, 3] 3. Explanation:    - Step 1:\
      \ Calculate ratios:      - Project 31: Utility = 38, Cost = 8,000, Ratio =\
      \ 38 / 8,000 \u2248 0.00475      - Project 2: Utility = 39, Cost = 13,000,\
      \ Ratio = 39 / 13,000 \u2248 0.003      - Project 23: Utility = 36, Cost =\
      \ 40,000, Ratio = 36 / 40,000 \u2248 0.0009      - Project 13: Utility = 40,\
      \ Cost = 24,000, Ratio = 40 / 24,000 \u2248 0.00167      - Project 14: Utility\
      \ = 41, Cost = 25,000, Ratio = 41 / 25,000 \u2248 0.00164      - Project 3:\
      \ Utility = 48, Cost = 27,000, Ratio = 48 / 27,000 \u2248 0.00178      - Others\
      \ have lower ratios or zero utility.    - Step 2: Sort: [31, 2, 23, 13, 14,\
      \ 3,...]    - Step 3: Add Project 31 (Cost = 8,000, Remaining Budget = 500,000\
      \ - 8,000 = 492,000).    - Step 4: Add Project 2 (Cost = 13,000, Remaining\
      \ Budget = 492,000 - 13,000 = 479,000).    - Step 5: Add Project 23 (Cost =\
      \ 40,000, Remaining Budget = 479,000 - 40,000 = 439,000).    - Step 6: Add\
      \ Project 13 (Cost = 24,000, Remaining Budget = 439,000 - 24,000 = 415,000). \
      \   - Step 7: Add Project 14 (Cost = 25,000, \hl{Remaining Budget = 415,000 - 25,000
      \ = 390,000).    - Step 8: Add Project 3 (Cost = 27,000, Remaining Budget =
      \ 390,000 - 27,000 = 363,000).    - Step 9: Next project exceeds the budget.}\
      \ Stop. 4. Return selected projects ( W = [31, 2, 23, 13, 14, 3] ).  "}
        \end{minipage}
        }
 \caption{Output of \llamaThreeSeventyB\ when presented with a PB instance of type \texttt{PPI} and tasked with executing the UG algorithm. The model stops prematurely, with the baseline UG allocation consuming USD 481,400 out of the USD 500,000 budget while the model's allocation only consumes USD 137,000. Additionally, the model "loosely" follows the greedy allocation order. The actual greedy order is $[3, 14, 13, 12, 2, 31, 21, 23, 33, 41, 15, 50, 42, 43, 47, 22, 36, 5, 37, 8]$.}
    \label{fig:output_greedy_base}
\end{figure*}

%%%%%%%%%%%%%%%%%%%%%%%%%%%%%%%
%%%%%%%%%%%%%%%%%%%%%%%%%%%%%%%
%%%%%%%%%%%%%%%%%%%%%%%%%%%%%%%

%% file: sample_outputs/output_greedy_good.tex
%%%%%%%%%%%%%%%%%%%%%%%%%%%%%%%
%%%%%%%%%%%%%%%%%%%%%%%%%%%%%%%
%%%%%%%%%%%%%%%%%%%%%%%%%%%%%%%
\begin{figure*}[h]
    \centering
        \fbox{%
        \begin{minipage}{\textwidth}
        \scriptsize
        \ttfamily
        \texttt{\{'allocation': [19, 26, 4, 51, 30], 'explanation':  "\#\#\# Step 1: Calculate the utility-to-cost ratio for each project.\
      \  - Project 19: Utility = 50, Cost = 40,000, Ratio = 50 / 40,000 = 0.00125\
      \  - Project 26: Utility = 43, Cost = 20,000, Ratio = 43 / 20,000 = 0.00215\
      \  - Project 4: Utility = 42, Cost = 60,000, Ratio = 42 / 60,000 = 0.0007  \
      - Project 51: Utility = 43, Cost = 105,000, Ratio = 43 / 105,000 \u2248 0.00041\
      \  - Project 30: Utility = 51, Cost = 250,000, Ratio = 51 / 250,000 \u2248\
      \ 0.000204  - Other projects have lower ratios.   Step 2: Sort projects by\
      \ this ratio in descending order: [26, 19, 4, 51, 30,...]   Step 3: Select\
      \ projects one by one, checking the budget constraint:  - Add Project 26 (Cost\
      \ = 20,000, Remaining Budget = 500,000 - 20,000 = 480,000).  - Add Project\
      \ 19 (Cost = 40,000, Remaining Budget = 480,000 - 40,000 = 440,000).  - Add\
      \ Project 4 (Cost = 60,000, Remaining Budget = 440,000 - 60,000 = 380,000).\
      \  \hl{- Add Project 51 (Cost = 105,000, Remaining Budget = 380,000 - 105,000 =
      275,000).  - Add Project 30 (Cost = 250,000, Remaining Budget = 275,000 -\
      \ 250,000 = 25,000).  - Next project exceeds the budget. Stop.}   Return selected\
      \ projects (W = [19, 26, 4, 51, 30]). \#\#\#"
        }
        \end{minipage}
    }
 \caption{Output of \qwen\ when presented with a PB instance of type \texttt{PPI} and tasked with executing the UG algorithm. The model does not stop prematurely, with the baseline UG algorithm's allocation \textbf{being identical} to the model's allocation. However, the greedy order is still "loose". The actual greedy order is $[30, 19, 35, 48, 26, 51, 4, 20, 38, 7]$.}
    \label{fig:output_greedy_good}
\end{figure*}

%%%%%%%%%%%%%%%%%%%%%%%%%%%%%%%
%%%%%%%%%%%%%%%%%%%%%%%%%%%%%%%
%%%%%%%%%%%%%%%%%%%%%%%%%%%%%%%

%% file: system_prompts/ppi.tex
\begin{figure*}
    \centering
    \fbox{%
        \begin{minipage}{\textwidth}
        \tiny
        You are a highly skilled social planner tasked with designing an optimal allocation of projects in a participatory budgeting process. Your objective is to ensure that the allocated projects maximize the overall social welfare while adhering to the given budget constraints and considering the preferences of the community members. \\

        Context of the Game: \\
        - The community has submitted projects that can potentially be funded. Each project has: \\
          - A name.
          - A cost.
          - A list of categories representing its impact areas (e.g., education, health, culture). \\
        - - Community members have voted for the projects they prefer. Their votes are represented as approval votes indicating which projects they support. \\
        - The total budget is limited, and only a subset of projects can be funded. \\ \\ 

        Your Goals:
        
        1. Maximize Social Welfare: Select projects that lead to the highest total satisfaction for the community.
        
        2. Respect the Budget: Ensure the total cost of the selected projects does not exceed the available budget.
        
        3. Fair Allocation: Strive for a balanced allocation that reflects the diverse preferences of the community.

        Information Provided to You:
        - A list of projects with their details (name, cost, and categories).
        
        - The budget limit for the allocation.
        
        - Voting data showing the preferences of community members.

        Instructions:
        1. Analyze the list of projects and their associated costs and categories.
        
        2. Consider the community member preferences to understand community preferences for the projects.
        
        3. Use the community member's preference to predict the projects the member would be interested in. Be wise in your predictions. This is an important step.
        
        4. You must use the following algorithm: You are tasked with implementing the Utilitarian Greedy (UG) algorithm for project selection. The utility of a project is: 
 
\#\#\# Definition:
The utility of a project is equal to the number of voters voting for that project, if the project is selected. Otherwise, zero.

\#\#\# Example:
Input:
Number of voters: 25

Output:
The utility of the project is 25 if selected, zero otherwise.

 Start with an empty set W. Repeatedly select the project p that maximizes the ratio of the p's utility to its cost. 
                If adding p to W keeps the total cost (sum of cost of all projects selected so far) within budget b, INCLUDE it in W; otherwise, EXCLUDE p and continue. You cannot exceed the budget.
                Stop when no projects remain and return W.

\#\# Output:
You must: (1) provide each community member's choice of projects; and (2) provide the budget allocation (that is the subset of projects allocated) based on the community choices.

Note: You must ensure that your allocation does not exceed the budget and explain how the selected projects contribute to social welfare.

\#

\# The example below will help you better understand the setup. 
\# Note how the community members vote for their projects by specifying the project IDs.
\# Note that Additional Data in the example to understand the voter utilitiy and the project that was selected within the budget constraint.

\# \#\#\# Example:

\{exemplar\_1\}

\{exemplar\_2\}

\#\#\#\#\#\#\#\#\#\#\#\#\#\#\#\#\#\#\#\#\#\#\#\#\#\#\#\#\#\#\#\#\#\#\#\#\#\#\#\# 
 
    \#\#\# Current Participatory Budgeting Game \\
    You are a highly skilled social planner tasked with designing an optimal allocation of projects in a participatory budgeting process. Your objective is to ensure that the allocated projects maximize the overall social welfare while adhering to the given budget constraints and considering the preferences of the community members.\\
    The Game Instance is: \\
    \#\#\#\#\# 

    \{pb\_instance\}
    
    \#\#\#\#\# 
 
    The available budget is:
    \{budget\}
    
    \#\#\#\#\# 

    Use the context provided for allocation for the Game Instance provided. 
    You MUST stay within the budget. The total project cost cannot exceed the budget.
    Do not add new rules to the setup. 

    \#\#\#\#
    Remember to use the greedy algorithm mentioned for project allocation.
    \#\#\#\# 

    The output format is:
    \{output\_data\_example\}
    
    \#\#\# 

    Remember to state your offer in the JSON format specified. 
    Do NOT say anything more. Do not include project tuples like ("51", 105000, 52). Only use Project IDs.     
    You MUST stay within the budget. The total project cost cannot exceed the budget.
    \#\#\#
        \end{minipage}
    }
    \caption{\texttt{PPI} \& \texttt{NLVPI}: System Prompt}\label{fig::system_prompt_ppi}
\end{figure*}

%% file: system_prompts/vrpi.tex
\begin{figure*}
    \centering
    \fbox{%
        \begin{minipage}{\textwidth}
        \tiny
        You are a highly skilled social planner tasked with designing an optimal allocation of projects in a participatory budgeting process. Your objective is to ensure that the allocated projects maximize the overall social welfare while adhering to the given budget constraints and considering the preferences of the community members. \\

        Context of the Game: \\
        - The community has submitted projects that can potentially be funded. Each project has: \\
          - A name.
          - A cost.
          - A list of categories representing its impact areas (e.g., education, health, culture). \\
        - Community members present. You have access to their attributes (e.g., sex, age, etc.).  
          Use this information to come up with the possible votes for the projects they prefer. 
          Their votes are represented as approval votes indicating which projects they support.
          Each community member can vote for a project by mentioning the project's ID. \\
        - The total budget is limited, and only a subset of projects can be funded. \\ \\ 

        Your Goals:
        
        1. Maximize Social Welfare: Select projects that lead to the highest total satisfaction for the community.
        
        2. Respect the Budget: Ensure the total cost of the selected projects does not exceed the available budget.
        
        3. Fair Allocation: Strive for a balanced allocation that reflects the diverse preferences of the community.

        Information Provided to You:

        - A list of projects with their details (name, cost, and categories).
        
        - The budget limit for the allocation.
        
        - Community members and their personal attributes.

        Instructions:
        
        1. Analyze the list of projects and their associated costs and categories.
        
        2. Consider the community member preferences to understand community preferences for the projects.
        
        3. Use the community member's preferene to predict the projects the member would be interested in. Be wise in your predictions. This is an important step.
        
        4. You must use the following algorithm: You are tasked with implementing the Utilitarian Greedy (UG) algorithm for project selection. The utility of a project is: 
 
\#\#\# Definition:
The utility of a project is equal to the number of voters voting for that project, if the project is selected. Otherwise, zero.

\#\#\# Example:
Input:
Number of voters: 25

Output:
The utility of the project is 25 if selected, zero otherwise.

 Start with an empty set W. Repeatedly select the project p that maximizes the ratio of the p's utility to its cost. 
                If adding p to W keeps the total cost (sum of cost of all projects selected so far) within budget b, INCLUDE it in W; otherwise, EXCLUDE p and continue. You cannot exceed the budget.
                Stop when no projects remain and return W.

\#\# Output:
You must: (1) provide each community member's choice of projects; and (2) provide the budget allocation (that is the subset of projects allocated) based on the community choices.

Note: You must ensure that your allocation does not exceed the budget and explain how the selected projects contribute to social welfare.

\#

\# The example below will help you better understand the setup. 
\# Note how the community members vote for their projects by specifying the project IDs.
\# Note that Additional Data in the example to understand the voter utilitiy and the project that was selected within the budget constraint.

\# \#\#\# Example:

\{exemplar\_1\}

\{exemplar\_2\}

\#\#\#\#\#\#\#\#\#\#\#\#\#\#\#\#\#\#\#\#\#\#\#\#\#\#\#\#\#\#\#\#\#\#\#\#\#\#\#\# 
 
    \#\#\# Current Participatory Budgeting Game \\
    You are a highly skilled social planner tasked with designing an optimal allocation of projects in a participatory budgeting process. Your objective is to ensure that the allocated projects maximize the overall social welfare while adhering to the given budget constraints and considering the preferences of the community members.\\
    The Game Instance is: \\
    \#\#\#\#\# 

    \{pb\_instance\}
    
    \#\#\#\#\# 
 
    The available budget is:
    \{budget\}
    
    \#\#\#\#\# 

    Use the context provided for allocation for the Game Instance provided. 
    You MUST stay within the budget. The total project cost cannot exceed the budget.
    Do not add new rules to the setup. 

    \#\#\#\#
    Remember to use the greedy algorithm mentioned for project allocation.
    \#\#\#\# 

    The output format is:
    \{output\_data\_example\}
    
    \#\#\# 

    Remember to state your offer in the JSON format specified. 
    Do NOT say anything more. Do not include project tuples like ("51", 105000, 52). Only use Project IDs.     
    You MUST stay within the budget. The total project cost cannot exceed the budget.
    \#\#\#
        \end{minipage}
    }
    \caption{\texttt{VRPI}: System Prompt}\label{fig::system_prompt_vrpi}
\end{figure*}

%% file: system_prompts/conflict.tex
\begin{figure*}
    \centering
    \fbox{%
        \begin{minipage}{\textwidth}
        \tiny
You are a highly skilled social planner tasked with designing an optimal allocation of projects in a participatory budgeting process. 
Your objective is to ensure that the allocated projects maximize the overall social welfare while adhering to the given budget constraints and considering the preferences of the community members.
The social welfare is the sum of the utilities for each voter, for each project allocated. 
The utility of a project is: 

\#\#\# Definition:
The utility of a project is equal to the number of voters voting for that project, if the project is selected. Otherwise, zero.

\#\#\# Example:
Input:
Number of voters: 25

Output:
The utility of the project is 25 if selected, zero otherwise.

\#\#\# Key Instructions:

1. **Maximize Social Welfare**: The utility of each project depends on how much it is supported by the community. The social welfare is the sum of the utilities for each voter, for each project allocated. 

2. **Respect the Budget**: The total cost of the selected projects **must not exceed the available budget**. If your selection exceeds the budget, you must adjust and try again. This is crucial to the solution.

3. **Fair and Efficient Allocation**: Strive for a balanced allocation that reflects the preferences of the community. Choose projects that are highly valued by voters, but ensure that the total cost is within the budget.

4. **Respect the Project Conflicts**: Some projects may be in conflict with each other. Ensure that you do not select conflicting projects together.

\#\#\# Your Approach:

- **Projects Details**: Each project has a name, cost, and associated categories. Review the projects carefully

- **Voter Preferences**: Community members vote for the projects they prefer. Higher utility is generated when projects that are more popular are selected.

- **Budget**: The total cost of selected projects must **not exceed** the available budget. After selecting projects, check and ensure the sum of costs is within the allowed limit.

- **CONFLICTS**: The projects that have conflicts. Example: \{conflict\_example\}

\#\#\# Process:

1. Analyze the community’s preferences based on votes for each project.

2. Select projects that maximize social welfare (voter satisfaction) and remain within the budget.

3. The total cost of the selected projects should not exceed the budget. If it does, re-evaluate your choices.

4. Provide a detailed explanation of your selected projects and calculate the total cost to ensure it does not exceed the budget.

\#\#\# Reminders:

- **Utility of a project**: The utility is the sum of the preferences (votes) for that project, where higher votes lead to a higher utility.

- **Total Budget**: Do not exceed the available budget at any stage.

\#\#\# Output Requirements:

- **Initial Selection**: List the IDs of the projects you initially selected in **string format** (e.g., "['1', '3', '5']"). Ensure this list only includes projects whose total cost does not exceed the budget.

- **Final Selection**: After revising the allocation (if necessary), output the final set of projects that adhere to the budget constraint.

- **Explanation**: Provide a rationale for your selections, discussing the community preferences and the budget constraint. Explain any adjustments made after the initial selection.

\#\#\# Process Example:

- If your initial selection exceeds the budget, revise your choices and explain the changes you made. Ensure that the final allocation is within the budget and clearly explain how the selected projects maximize social welfare.

Once you have your final selection, summarize the projects you chose, explain why, and confirm the total cost does not exceed the budget or you do NOT select conflicting projects together.

\#\#\#\#\# Example:

\{exemplar\_1\}

\{exemplar\_2\}

\#\#\#\#\#\#\#\#\#\#\#\#\#\#\#\#\#\#\#\#\#\#\#\#\#\#\#\#\#\#\#\#\#\#\#\#\#  

    \#\#\# Current Participatory Budgeting Game
    
    You are a highly skilled social planner tasked with designing an optimal allocation of projects in a participatory budgeting process. Your objective is to ensure that the allocated projects maximize the overall social welfare while adhering to the given budget constraints and considering the preferences of the community members.
    
    The Game Instance is: 
    \#\#\#\#\# 
    
    \{pb\_instance\}
    
    \#\#\#\#\# 
    
    The available budget is:
    \{budget\}

    \#\#\#\#\# 
    Use the context provided for allocation for the Game Instance provided. 
    Do not add new rules to the setup. 
    
    \#\#\#\# 
    
    The output format is:
    \{output\_data\_example\}

    \#\#\# 
    Remember to state your offer in the JSON formatted specified. Do not say anything more.
    \#\#\#

        \end{minipage}
    }
    \caption{PB with Project Conflicts: System Prompt}\label{fig::system_prompt_conflict}
\end{figure*}